\documentclass{article} 
\usepackage{nips15submit_e,times}
  
 \usepackage{hyperref}
\usepackage{url}
\usepackage{graphicx}
\usepackage{epstopdf}

\graphicspath{{speedup/},{cocoa/}}
\usepackage[colorinlistoftodos,bordercolor=orange,backgroundcolor=orange!20,linecolor=orange,textsize=scriptsize]{todonotes}

\usepackage{algorithm}
\usepackage{algorithmic}
\title{Distributed Mini-Batch SDCA}

\author{  
Martin Tak\'a\v{c} \\
Lehigh University \\
Bethlehem, PA, USA \\
\texttt{Takac.MT@gmail.com} \\
\And
Peter Richt\'arik \\
University of Edinburgh \\
Edinburgh, UK \\
\texttt{Peter.Richtarik@ed.ac.uk	} \\
\AND
Nathan Srebro \\
Toyota Technological Institute at Chicago \\
Chicago, IL, USA \\
\texttt{nati@ttic.edu} \\
}

\nipsfinalcopy 

\newcommand{\setn}{\langle n\rangle}
\newcommand{\Exp}{\mathbb{E}}                      

\usepackage{multicol}

\usepackage{latexsym, amsmath,amssymb, amsthm}
\usepackage{amsmath, amsthm, amssymb}
\usepackage{amsmath}
\usepackage{graphicx}
\usepackage{subfigure}
\usepackage{color}

\DeclareMathOperator{\dom}{dom}         

\newcommand{\ve}[2]{\left\langle #1 ,  #2 \right\rangle}    
\newcommand{\eqdef}{:=}
\newcommand{\R}{\mathbb{R}}                      
\newcommand{\Prob}{\mathbb{P}}                   
\newcommand{\E}{\mathbb{E}}                      
\newcommand{\N}{n}                               

\newcommand{\C}{C}

\newcommand{\diag}{\mathbf{diag}}

\newcommand{\hatS}{\hat S}

\newcommand{\calG}{\mathcal{G}}

\newcommand{\calH}{\mathcal{H}}

\newcommand{\x}{x}

\newcommand{\vv}{v}
\newcommand{\V}{v}
\newcommand{\T}{ {\bf T}}
\newcommand{\X}{X}

\newcommand{\y}{y}
\newcommand{\alf}{\alpha}

\newcommand{\vu}{u}

\newcommand{\w}{w}
\newcommand{\vt}{t}

\newcommand{\vsubset}[2]{#1_{[#2]}}

\newcommand{\vc}[2]{#1^{(#2)}}                   

\usepackage[colorinlistoftodos,bordercolor=orange,backgroundcolor=orange!20,linecolor=orange,textsize=scriptsize]{todonotes} 
 
\newcommand{\removed}[1]{}
\newcommand{\norm}[1]{\left\lVert{#1}\right\rVert}

\newcommand{\bP}{\mathcal{P}}
\newcommand{\bD}{\mathcal{D}}

\theoremstyle{plain}
\newtheorem{theorem}{Theorem}
\newtheorem{lemma}[theorem]{Lemma}

\theoremstyle{definition}
\newtheorem{definition}[theorem]{Definition}

\begin{document}

\maketitle

\begin{abstract} 
We present an improved analysis of mini-batched stochastic dual
coordinate ascent for regularized empirical loss minimization
(i.e.~SVM and SVM-type objectives).  Our analysis allows for flexible
sampling schemes, including where data is distribute across machines,
and combines a dependence on the smoothness of the loss and/or the
data spread (measured through the spectral norm).
\end{abstract}

\section{Introduction}

Stochastic optimization approaches have significant theoretical and
empirical advantages in training linear Support Vector Machines (SVMs)
and other regularized loss minimization problems, and are often the
methods of choice in practice.  Such methods use a single, randomly
chosen, training example at each iteration.  In the context of SVMs,
many variations of stochastic gradient descent (SGD) have been
suggested, based on primal stochastic gradients (e.g. Pegasos
\cite{pegasos}, NORMA \cite{zhang}, SAG \cite{schmidt2013minimizing},
MISO \cite{mairal2014incremental}, S2GD \cite{konevcny2013semi}, mS2GD
\cite{konevcny2014ms2gd} and Prox-SVRG
\cite{johnson2013accelerating,nitanda2014stochastic}).  In this paper
we focus on SDCA---stochastic dual coordinate ascent---which is based on improvements to the dual problem, again considering only a single
randomly chosen training example, and thus only a single randomly
chosen dual variable, at each iteration
\cite{dcd,richtarik,lu2013complexity}.  Especially when accurate
solutions are desired, SDCA has better complexity
guarantee, and often performs better in practice than SGD
\cite{dcd,ShalevShawartzZhang}.

The inherent sequential nature of such approaches becomes a
problematic limitation in parallel and distributed settings as
the predictor must be updated after each training point is processed,
providing very little opportunity for parallelization.  A popular
remedy is to use {\em mini-batches}: the use several training
points at each iteration, calculating the update
based on each point separately and aggregating the updates.  The
question is then whether basing each iteration on several points can
indeed reduce the number of required iterations, and thus yield
parallelization speedups.  

For SGD with a non-smooth loss, mini-batching does not reduce the
number of worst-case required iterations and thus does not allow
parallel speedups in the worst case.  However, when the loss function
is smooth, mini-batching can be beneficial and linear speedups can be
obtained, even when the mini-batch sizes scales polynomially with the
total training set size
\cite{dekel2012optimal,agarwal2011distributed,cotter2011better}.
Alternatively, even for non-smooth loss, linear speedups can also be
ensured if the data is reasonably well-spread, as measured by the
spectral norm of the data, as long as the mini-batch size is not
larger then the inverse of this spectral norm
\cite{TRN:minibatchICML}.  

For SDCA, using a mini-batch corresponds to updating multiple
coordinates concurrently and independently.  If appropriate care is
taken with the updates (see Section \ref{garanteSpeedup}), then using a mini-batch
size as large as the inverse spectral norm leads to a reduction in the
number of iterations, and allows linear parallelization speedups, even
when the loss is non-smooth \cite{Bradley,TRN:minibatchICML}.  This
parallels the SGD mini-batch analysis for non-smooth loss.  But can
mini-batching also be beneficial for SDCA with smooth losses and
without a data-spread (spectral norm) assumptions, as with SGD?  What
mini-batch sizes allow for parallel speedups?  In this paper we answer
these questions and show that as with SGD, when the loss function is
smooth, using mini-batches with SDCA yields a linear reduction in
the number of iterations and thus allows for linear parallelization
speedups, up to similar polynomial limits on the mini-batch size.
Furthermore, we provide an analysis that combines the benefits of
smoothness with the data-dependent benefits of a low spectral norm,
and thus allows for even large mini-batch sizes when the loss is
smooth {\em and} the data is well-spread.

Another issue that we address is the way mini-batches are sampled.
Straight-forward mini-batch analysis, including previous analysis of
mini-batch SDCA \cite{TRN:minibatchICML}, assume that at each
iteration we pick a mini-batch of size $b$ uniformly at random from
among all subsets of $b$ training examples.  In practice, though, data
is often partitioned between $C\leq b$ machines, and at each iteration
$b/C$ points are samples from each machine, yielding a mini-batch that
is not uniformly distributed among all possible subsets (e.g.~we have
zero probability of using $b$ points from the same machine as a
mini-batch).  Other architectural restrictions might lead to different
sampling schemes.  The analysis we present can be easily applied to
different sampling schemes, and in particular we consider distributed
sampling as described above and show that essentially the same
guarantees (with minor modification) hold also for this more realistic
sampling scheme.

Finally, we compare our optimization guarantees to those recently
established for CoCoA+ \cite{ma2015adding}.  CoCoA+ is an alternative dual-based
distributed optimization approach, which can be viewed as including
mini-batch SDCA as a special case, and going beyond SDCA to
potentially more powerful optimization.  At each iteration of CoCoA+,
several groups of dual variables are updated.  We focus on CoCoA+SDCA,
where each group is updated using some number of SDCA iterations.
When each group consists of a single variable, this reduces exactly to
mini-batch SDCA.  Allowing for multiple SDCA iterations on larger
groups of variables yields a method that is more computationally
demanding that mini-batch SDCA, and intuitively should be better than
SDCA (and does appear better in practice).  However, we show that our
mini-batch SDCA analysis strictly dominates the CoCoA+ analysis: that
is, with the same number of total dual variables updated per
iteration, and thus less computation, our mini-batch SDCA guarantees
are strictly better than those obtained for CoCoA+.  Mini-batch SDCA
is thus a simpler, computationally cheaper method, with better {\em
guarantees} than those established for CoCoA+.

Although SDCA is a dual-method, improving the dual at each iteration,
following the analysis methodology of \cite{ShalevShawartzZhang}, all
our guarantees are on the duality gap, and thus on the {\em primal}
sub-optimality, that is on the actual regularized error we care about.

\section{Setup and Preliminaries}

We consider the problem of minimizing the regularized empirical loss 
\begin{align}\tag{P}\label{eq:P}
 \min_{\w\in\R^d} \bP(\w):=\tfrac1n \textstyle{\sum}_{i=1}^n \phi_i(\w^T \x_i) +\tfrac\lambda2 \|\w\|^2,
\end{align}
where $\x_1,\ldots,\x_n\in\R^d$ are given training examples,
$\lambda>0$ is a given regularization parameter and
$\phi_i:\R\to\R$ are given non-negative convex loss functions
that already incorporate the labels (e.g.~$\phi_i(z)=\phi(y_i z)$
where $y_i\in\pm 1$ are given labels). Instead of solving
\eqref{eq:P}, we  solve the dual \cite{ShalevShawartzZhang}
\begin{align}\tag{D}\label{eq:dualOfP}
 \max_{\alf \in\R^n} \bD(\alf):=-\tfrac1n \textstyle{\sum}_{i=1}^n \phi_i^*(-\alpha_i) -
          \tfrac {\lambda}{2}  \|\tfrac1{\lambda n} \X^T\alf   \|_2^2,
\end{align}
where $\phi_i^*(u): \R \to \R$ is the convex conjugate of $\phi_i$
defined in the standard way as $\phi_i^*(u) = \max_z ( z u - \phi(z))$
and $\X=[\x_1,\dots, \x_n]^T \in \R^{n \times d}$ is the data matrix,
where each row corresponds to one sample and each column corresponds
to one feature.  If $\alf^*$ is a dual-optimum of \eqref{eq:dualOfP}
then $\w^*=\tfrac1{\lambda n} \X^T\alf^*$ is a primal-optimum of
\eqref{eq:P}.  We therefor consider the mapping
$\w_\alf=\tfrac1{\lambda n} \X^T\alf$ and define the {\em duality gap}
of a feasible $\alf\in\dom(\bD)$ as:
\begin{equation}\label{eq:dualityGap}\tag{G}
 \calG(\alf) 
    \eqdef 
    \bP(\w_\alf) - \bD(\alf).
\end{equation}

\paragraph{Stochastic Dual Coordinate Ascent (SDCA)}

SDCA is a coordinate ascent algorithm
optimizing 
the dual  
\eqref{eq:dualOfP}.
At $t$-th iteration of SCDA
a coordinate 
$i \in \setn := \{1,2,\dots,n\}$
is chosen at random
and then 
a new iteration 
is obtained by updating only the $i$-th coordinate and
keeping all other
coordinates of $\alf$ unchanged i.e. 
$\vc{\alpha}{t+1} = \vc{\alpha}{t}
+ \vc{\Delta\alpha}{t}_i e_i$, 
where 
\begin{equation}
\label{eq:serialUpdaterule}
\vc{\Delta\alpha}{t}_i
=\arg\min_{\delta \in \R}\mathcal{D}(\vc{\alpha}{t}+ \delta e_i).
\end{equation}
 
\paragraph{Assumptions on Loss Function}

We analyze mini-batched SDCA under one of two different assumptions
on the loss functions: that they are $L$-Lipschitz continuous (but
potentially non-smooth), or that they are $(1/\gamma)$-smooth.
Formally: i) {\bf $L$-Lipschitz continuous loss:}
 $\forall i, \forall a,b \in \R$ we have 
$| \phi_i(a) - \phi_i(b) | \leq L |a-b|$, ii) {\bf $(1/\gamma)$-smooth loss:} Each loss function $\phi_i$ if differentiable and its derivative is
$(1/\gamma)$-Lipschitz continuous, i.e.~$\forall a,b\in\R$ we have
$| \phi_i'(a) - \phi_i'(b) | \leq \tfrac1\gamma |a-b|$; iii) We also assume $\phi_i$ are non-negative and that $\phi_i(0)\leq 1$
for all $i$.

For a positive vector $\vv=(v_1,\dots,v_n)^T>0$ we define a weighted
Euclidean norm $\|\alpha\|_\V^2 = \sum_{i=1}^n v_i \alpha_i^2$.
Instead of assuming the data is uniformly bounded, we will frequently
refer to the weighted norm on $\R^n$ with weights proportional to the
squared magnitudes, i.e. $v_i \sim \|x_i\|^2$.


\section{Mini-Batched SDCA}\label{sec:mbsdca}

At each iteration of mini-batched SDCA, a subset $S \subseteq \setn$
of the coordinates is chosen at random (see below for a discussion of
the sampling distribution) and a new dual iterate is obtained by
independently updating only the chosen coordinates.  Since each
coordinate is updated independently, mini-batch SDCA is amenable to
parallelization. 

The na\"ive approach is to use the same update rule for each
coordinate as in serial case: the update is then given by
$\vc{\alpha}{t+1} = \vc{\alpha}{t} + \sum_{i \in S} \vc{\Delta
  \alpha}{t}_i$ where $\vc{\Delta \alpha}{t}_i$ is given by
\eqref{eq:serialUpdaterule}.  Such a na\"ive approach could be fine if
the mini-batch size is very small and the data is ``spread-out''
enough \cite{Bradley}.  However, more
generally, not only might such a mini-batch iteration not be better
than an iteration based on only a single point, but such a na\"ive
mini-batch update might actually be much worse.  In particular, it is
easy to construct an example with just two examples where a na\"ive
mini-batch approach will never reach the optimum solution, and
diverging behavior frequently occurs in practice on real data sets
\cite{TRN:minibatchICML}.  The problem here is that the independent
updates on multiple similar points might combine together to
``overshoot'' the optimum and hurt the objective.

An alternative that avoids this problem is to average the updates
instead of adding them up, $\vc{\alpha}{t+1} = \vc{\alpha}{t}
+\frac1{|S|} \sum_{i \in S} \vc{\Delta \alpha}{t}_i$
\cite{jaggi2014communication, yang2013theoretical,yang2013trading},
but such an update is overly conservative: it is not any better than
just updating a single dual variable, and cannot lead to
parallelization speedups.
Following \cite{TRN:minibatchICML}, the approach we consider here is to use a summed update
$\vc{\alpha}{t+1} = \vc{\alpha}{t} + \sum_{i \in S} \vc{\Delta
  \alpha}{t}_i$, where the independent updates $\vc{\Delta
  \alpha}{t}_i$ are derived from a relaxation of the dual:
\begin{equation}\label{eq:relaxedDelta}
\vc{\Delta \alpha_i}{t}  = \arg\max_\delta -\phi_i^*(- \alpha_i-\delta)
+ \tfrac1{2 \lambda n} v_i \delta^2 - \w_\alf^T x_i \delta
\end{equation}
When $v_i=\norm{\x_i}^2$, the update exactly agrees with the
dual-optimizing update \eqref{eq:serialUpdaterule}.  But as we shall see,
when larger mini-batches are used, larger values of $v_i$ are
required, resulting in smaller steps.  The update
\eqref{eq:relaxedDelta} generalizes \cite{TRN:minibatchICML} where a
single parameter $v_i=v$ was used---here we allow $v_i$ to vary
between dual variables, accommodating differences in $\norm{\x_i}$.

To summarize, the mini-batch SDCA algorithm we consider takes as input
data $X$, loss functions $\phi_i$, a distribution over subsets $S
\subseteq \setn$, which we will refer to as the random sampling
$\hatS$, and a weight vector $v$, and proceeds as shown on Algorithm \ref{alg:SafeMiniBatching}.
\begin{algorithm}
 \begin{algorithmic}[1]
 
  \STATE {\bf Input:} $\X, \y, \hatS, v$ 
  \STATE set $\vc{\alf}{0}={\bf 0}\in\R^n$
  \FOR {$t=0,1,2,\dots$}
    \STATE 
    \label{stp:subset}
    choose $S_t$
    according the distribution $\hatS$    \STATE set $\vc{\alf}{t+1}=\vc{\alf}{t}$; $w_\alf=\frac1{\lambda n} X^T \vc{\alpha}{t}$
    \FOR {$i\in S_t$ {\bf in parallel}}
      \STATE
      \label{stp:compUpdate}
$\vc{\Delta \alpha_i}{t}
 =
\arg\max_\delta
-   \phi_i^*(- \alpha_i-\delta)  
-  
\tfrac1{2 \lambda n}
v_i \delta^2     
     -         \w_\alf^T x_i \delta
$

      \STATE
      $\vc{\alpha_i}{t+1}=\vc{\alpha_i}{t}+\vc{\Delta \alpha_i}{t}$ 
    \ENDFOR
  \ENDFOR  
 \end{algorithmic}

 \caption{mSDCA: minibatch Stochastic Dual Coordinate Ascent}
 \label{alg:SafeMiniBatching}
\end{algorithm}

We will refer to several sampling distributions $\hatS$, yielding different
variants of mini-batch SDCA:
\\ {\bf Serial SDCA.}  
$\hatS$ is a uniform distribution over singletons.  That is, $S_t$
contains a single coordinate chosen uniformly at randomly.  Setting
$v_i=\norm{\x_i}$ yields standard SDCA.\\
{\bf Standard Mini-batch SDCA.}
$\hatS$ is a uniform distribution over subsets of size $b$. 
{\bf Distributed SDCA.}  Consider a setting with $C$ machines,
$n$ total data points and a mini-batch size $b$, where for simplicity
$n$ and $b$ are both integer multiples of $C$.  For a partition of the
$n$ coordinates into $C$ equal sized subsets $\{P_c\}_{c=1}^\C$,
consider the following sampling distribution $\hatS$: for each
$c=1..\C$, choose a subset $S^c\subset P_c$ uniformly and
independently at random among all such subsets of size $b/C$, and then
take their union.  We refer to such a sample as a $(\C,b)$-distributed
sampling.  Such a sampling is  suitable in a distributed environment
when $n$ samples are equally partitioned over $\C$ computational nodes in a cluster \cite{richtarik2013distributed,marecek2014distributed}.
When $\C=1$ we obtain the Standard Mini-batch sampling.

The main question we now need to address is what weights $v_i$ are
suitable for use with each of the above sampling schemes, and what
optimization guarantee to they yield.  To answer this question, in the
next Section we will introduce the notion of Expected Separable Overapproximations.
\section{Expected Separable Overapproximation}
\label{sec:ESO}
 \label{sec:computingUpdate:ESO}
 
In this Section we will make use of  the Expected Separable Overapproximation (ESO) theory 
introduced in 
\cite{richtarikBigData}
and further extended 
e.g. in
\cite{
marecek2014distributed,
richtarik2013distributed,
qu2014coordinate}.

\subsection{Motivation} 
Consider the $t$-th iteration of mini-batch SDCA.
Our current iterate is $\vc{\alpha}{t}$
and we have chosen a set $S_t$
of coordinates which we will update in current iteration.
We need to compute the updates to those coordinates, i.e.
$\forall i \in S_t$ we need to compute
$\vc{\Delta \alpha_i}{t}$.
Maybe the natural way how to define the updates would be to define them such that $D(\vc{\alpha}{t+1})$ is as large as possible, i.e.
that we maximize
$D(\vc{\alpha}{t}
+\sum_{i \in S_t}
 \vc{\Delta \alpha_i}{t} e_i)$.
However, this e.g. for hinge loss would lead to a QP, hence the computation cost would be substantial. The main disadvantage of this approach is the fact that the updates for different coordinates are dependent on each other, i.e. the value of $\vc{\Delta \alpha_i}{t}$ depends on all coordinates in $S_t$.
This make it hard to parallelize.  
Considering the fact that $S_t$ is a random set, maybe one would like to define the updates so that 
the updates doesn't depend on current choice of $S_t$ and that they maximize the expected value of $D$ at next iteration. In this case we are facing following maximization problem
\begin{equation}
\max_{t \in \R^n} \Exp[D(\vc{\alpha}{t}
+ \vsubset{t}{S_t})],
\label{eq:afjoiewof3w4e3}
\end{equation}
 where
$\vsubset{t}{S_t}$ is a masking operator setting all coordinates of $t$ which are not in set $S_t$ to zero, i.e.
$(\vsubset{t}{S_t})_i = t_i$ if $i \in S_t$
and $(\vsubset{t}{S_t})_i=0$ otherwise. The 
expectation in \eqref{eq:afjoiewof3w4e3} is considered over the distribution $\hatS$.
After we get the optimal solution $t^*$
of \eqref{eq:afjoiewof3w4e3}
we can define
$\vc{\Delta \alpha_i}{t} = t^*_i$
for all $i\in S_t$.
Therefore 
$\vc{\alpha}{t+1}
 = \vc{\alpha}{t}+ 
 \sum_{i \in S_t} \vc{\Delta \alpha_i}{t}
  = \vc{\alpha}{t}+\vsubset{t^*}{S_t}$.
However, now the problem  
\eqref{eq:afjoiewof3w4e3}
is even more complicated.
The remedy is to replace
$\Exp[D(\vc{\alpha}{t}
+ \vsubset{t}{S_t})]$
by its separable lowerbound. Then due to the fact that it will be separable, the update for any coordinate $i$ will be independent on the  other coordinates in $S_t$
and moreover, the updates will be obtained by solving 1D problem.

\subsection{Lower-bound} 
 
Let us first state the definition of ESO.
\begin{definition}[Expected Separable Overapproximation \cite{richtarikBigData}]\label{def:ESO}
Assume that sampling
$\hatS$  has uniform marginals.
Then we say that function $f$ admits $v$-ESO with respect to the sampling $\hatS$ if 
 $\forall x,t \in \R^\N$ we have
\begin{equation}\label{eq:ESO}
  \Exp[ f(\alf+\vsubset{\vt}{\hatS})] \leq f(\alf) + \tfrac{\Exp[|\hatS|]}{n} 
   (\ve{ \nabla f(\alf)}{\vt}
  + \tfrac{1}{2} \|\vt\|_\V^2 ).
\end{equation}
\end{definition}
Let us now just assume that
we can find such a vector $v$ such that
\eqref{eq:ESO} holds (we show how to find $v$ in Section \ref{sec:ESO:computingV})
and we now show how to derive  
the lowerbound of 
$\Exp[D(\vc{\alpha}{t}
+ \vsubset{t}{S_t})]$.
If we write \eqref{eq:ESO}
for a particular choice of $f$, namely for
$f(\alf) =   \|\frac1{\lambda n}
\X^T \alf \|_2^2$ we obtain
\begin{align}
\label{eq:asdpfwafdwad}
 \Exp [
  \|\tfrac1{\lambda n}
\X^T (\alf+\vsubset{\vt}{\hatS}) \|_2^2
 ]
&\overset{\eqref{eq:ESO}}{\leq}
 \|\w_\alf  \|^2 
+ \tfrac{\Exp[|\hatS|]}{n}
   (         
        \|\tfrac1{\lambda n} \vt \|^2_\V
     +    \tfrac2{\lambda n}   \vt^T  \X  \w_\alf  ). 
\end{align}
Now we can derive the expected lowerbound of $\bD$ as follows
\begin{align}
\nonumber
\Exp [ \bD (\alf + \vsubset{\vt}{\hatS}   )  ] 
&\overset{\eqref{eq:dualOfP}}{=}
\Exp [
-\tfrac1n \textstyle{\sum}_{i=1}^n \phi_i^* (- (\alf + \vsubset{\vt}{\hatS}   )_i ) 
 ] 
-\Exp [
          \tfrac {\lambda}{2}  \|\tfrac1{\lambda n} \X^T (\alf + \vsubset{\vt}{\hatS}   )  \|_2^2
  ] 
\\
&\overset{\eqref{eq:asdpfwafdwad}}{\geq}
\nonumber
-\tfrac{\lambda}{2}  
           \|\w_\alf  \|^2
-
\tfrac{\Exp[|\hatS|]}{n}
  \tfrac1n \textstyle{\sum}_{i=1}^n \phi_i^* (-  \alf_i -  \vt_i     ) 
  -
 (1-\tfrac{\Exp[|\hatS|]}{n} )  
 \tfrac1n \textstyle{\sum}_{i=1}^n \phi_i^* (-  \alf_i  )  
\\ &\quad  
- \tfrac{\Exp[|\hatS|]}{n}
\tfrac\lambda2
   (         
         \|\tfrac1{\lambda n} \vt \|^2_\V
     +    \tfrac2{\lambda n}   \vt   ^T  \X  \w_\alf  ),
 \label{eq:sadfosapfdacdefwa}
\end{align}

where     in the first inequality for the first part we have used the fact that the function is separable (see Theorem 4 in \cite{richtarikBigData}).
If we  define
\begin{align} 
 \calH(\vt,\alf) &:= -\tfrac1n \textstyle{\sum}_{i=1}^n \phi_i^*(-(\alpha_i+t_i)) - \tfrac\lambda2 
       \|\w_\alf  \|^2  
-  \tfrac\lambda2 \|\tfrac1{\lambda n} \vt \|^2_\V     
     -   \tfrac1 { n}   \vt   ^T  \X  \w_\alf,
     \label{eq:H_definition}      
\end{align}
then it is easy to see that
we can find a separable (in $t$) expected lower approximation of $\bD$, i.e. it holds  
$\forall \alf, \vt\in\R^\N$ that 
$
 \Exp [ \bD (\alf + \vsubset{\vt}{\hatS}   )  ] \overset{\eqref{eq:sadfosapfdacdefwa}}{\geq} \tfrac b\N \calH(\vt,\alf) + \left(1-\tfrac b\N\right) \bD(\alf),
$
where $b \eqdef \Exp[|\hatS|]$ is the average number of mini-batch.
Now let us note again that 
it is very hard to maximize 
 $\Exp\left[ \bD\left(\alf + \vsubset{\vt}{\hatS}  \right) \right]$
 in $\vt$, but maximize of 
 $\calH(\vt,\alf)$ in $\vt$ is very simple, because this function is simple and separable in $\vt$.
It is also easy to verify that the steps in Algorithm \ref{alg:SafeMiniBatching} are maximizing $\calH$.

 \subsection{Computing ESO Parameter}
\label{sec:ESO:computingV}
In previous Section we have shown that
using ESO we can find a separable lowerbound of $\Exp [ \bD (\alf + \vsubset{\vt}{\hatS}   )]$.
However, we haven't explained how the ESO parameter (vector $v$) can be obtained.

In this Section we present some of the 
results obtained in literature 
\cite{richtarikBigData,
fercoq2013accelerated,
marecek2014distributed,
fercoq2014fast}
for formulas for computing vector $v$ for samplings described in
Section \ref{sec:mbsdca}.
Let us mention that all formulas are {\bf data dependent}.
Some of them involves
the spectral radius of following  matrix
$D^{-\frac12} \X\X^T D^{-\frac12} $,
where 
  $D = \diag(XX^T)$
which we will denote by $\sigma^2$, hence
$
\sigma^2 \eqdef  \max_{\alpha \in \R^n: \|\alpha\|=1}
\tfrac1n
\|\X^T D^{-\frac12} \alpha\|^2.
$ 
 Note that this 
can be in practise impossible to compute (we can estimate is using e.g. power method)
or we can use an upper-bound (derived
in  Lemma 5.4 \cite{fercoq2014fast})
 by 
 $ \omega =  \displaystyle \max_{i \in \setn} 
\tfrac1n 
  \tfrac{
\sum_{j=1}^d  \|  x_i \|_0  (\x_i^T e_j)^2
}{\sum_{j=1}^d (\x_i^T e_j)^2}$,
where by $\|x_i\|_0$ we have denoted a number of non-zero elements of $i$-th data point.

{\bf Serial SDCA.}  
In this simplest case we can define
$v_i = \|x_i\|^2$.
\\
{\bf Standard Mini-batch SDCA.}
In standard mini-batch we can choose
$v_i = (1+\frac{(b-1)(n \sigma^2-1)}{\max\{1,n-1\}})
 \|x_i\|^2$.
If the data matrix $X$ is sparse, we can define
$v_i = \sum_{j=1}^d  (\x_i^T e_j)^2  (1+\tfrac{(b-1)(\|x_i\|_0-1)}{n-1} )$.
\\
{\bf Distributed SDCA.} 
In  distributed case we can choose
$v_i=\frac{b}{b-C} (1+\tfrac{(b-C)(n \sigma^2-1)}{\max\{C,n-C\}}) \|x_i\|^2$,
provided that $b \geq 2C$
and $v_i = 
(1+ b     \sigma^2) \|x_i\|^2$
if $b=C$.
A simple upper-bound valid for any $b$ can be derived as follows
$v_i = 2(1 + b \sigma^2)\|x_i\|^2$.
\section{Convergence Guarantees}
\label{converence}

We are now ready to present optimization guarantees for Algorithm
\ref{alg:SafeMiniBatching} based on the ESO parameters studied in the
previous Section.  These theorems extends the serial case of
\cite{ShalevShawartzZhang} to mini-batch setting.  The Theorems are
based on weights $v$ are chosen such that $f(\alpha)=\|\frac1{\lambda
  n} X^T \alpha\|^2$ admits $v$-ESO for a sampling $\hatS$ used in the
Algorithm \ref{alg:SafeMiniBatching}.  Proofs are provided in the
supplemental material.
 
\begin{theorem}[$(1/\gamma)$-Smooth Loss]
  \label{thm:dualityGapForSmooth} If the losses are
  $(1/\gamma)$-smooth and $f(\alpha)=\|\frac1{\lambda n} X^T
  \alpha\|^2$ admits $v$-ESO for the sampling $\hatS$, then for a
  desired duality gap $\epsilon_\calG>0$, using Algorithm
  \ref{alg:SafeMiniBatching}, if we choose
\begin{equation}
\label{eq:asfewfwa}
 T \geq     \tfrac{ \|\vv\|_\infty}b
     (
       \tfrac1{\lambda   \gamma }
    +
    \tfrac{n}   { \|\vv\|_\infty}
     )
    \log  (
    \tfrac{ \|\vv\|_\infty}b
      (
     \tfrac1 {\lambda   \gamma }
    +
    \tfrac{n }  { \|\vv\|_\infty}
     )
    \tfrac1{\epsilon_\calG}
     )
\end{equation}
we have that
$\Exp[\bP(\w_T) - \bD(\alf_T) ] \leq \epsilon_\calG.$
To obtain
$\Exp[\bP(\bar \w) - \bD(\bar \alf) ] \leq \epsilon_\calG,$
it is sufficient to choose
$
 T_0   
    \geq 
    \tfrac{ \|\vv\|_\infty}b
     (
       \tfrac1{\lambda   \gamma }
    +
    \tfrac{n}   { \|\vv\|_\infty}
     )
    \log 
     (
    \tfrac{ \|\vv\|_\infty}b
     (
    \tfrac{1}{\lambda   \gamma }
    +
    \tfrac{n }{ \|\vv\|_\infty }
     ) \tfrac{1}{(T-T_0)\epsilon_\calG}
     ),
$
where 
\begin{equation}\label{eq:averageOfAlphaDefinition}
\bar \alf = \tfrac1{T-T_0}\textstyle{\sum}_{t=T_0+1}^{T-1} \vc{\alf}{t}. 
\end{equation}
Moreover, if 
$ \tilde T \geq \tfrac{ \|\vv\|_\infty+\lambda n \gamma }   {b \lambda   \gamma }
  \log (
  \tfrac{ \|\vv\|_\infty+\lambda n \gamma }   {b\lambda   \gamma }
  \tfrac{1}{
   \epsilon_\calG\rho}
    )
$ then  
$
\Prob(\bP(\w_{\tilde T}) - \bD(\alf_{\tilde T}) \leq \epsilon_\calG) \geq 1-\rho.
$
\end{theorem}

\begin{theorem}[$L$-Lipschitz Continuous Loss]
\label{thm:dualityGapForLipFunctions}
If the losses are
$L$-Lipschitz and $f(\alpha)=\|\frac1{\lambda n} X^T
  \alpha\|^2$ admits $v$-ESO for the sampling $\hatS$, then for a
  desired duality gap $\epsilon_\calG>0$, using Algorithm
  \ref{alg:SafeMiniBatching}, denoting $G= 4L^2 \frac{\sum_{i=1}^n
    v_i}{n}$, if we choose
\begin{align}
\label{eq:dualityRequirements}
T_0
&\geq t_0+
\tfrac{1}b
\left(
\tfrac{4 G} { \lambda \epsilon_\calG} -2n \right)_+,
&
T
&\geq
T_0 + \max\{\lceil \tfrac nb\rceil,\frac{1}{b}
\tfrac{ G}{\lambda \epsilon_\calG}\},  
\\
t_0 &\geq 
\label{eq:dualityRequirements3}
  \max(0,\lceil \tfrac nb \log(2\lambda n \vc{\epsilon_D}{0} /  G ) \rceil),
\end{align}
we have that
 $\Exp[\bP(\bar \w) - \bD(\bar \alf) ] \leq \epsilon_\calG,$ 
where 
$\bar\alf$ is defined in
\eqref{eq:averageOfAlphaDefinition}
 Moreover, when $t\geq T_0$, we have dual sub-optimality bound 
$
\Exp[\bD(\alf^*)-\bD(\vc{\alf}{t})]\leq \tfrac12\epsilon_\calG
$.
\end{theorem}


\section{Guarantees and Speedups for Specific Sampling Distributions}\label{garanteSpeedup}

Theorems \ref{thm:dualityGapForLipFunctions}
and \ref{thm:dualityGapForSmooth}
are stated in terms of ESO parameter $v$.  Let us now consider the
specific sampling distribution of interest.  Assume for simplicity
$\norm{\x_i}\leq 1$, and define
\begin{equation}
  \label{eq:beta}
\beta_{\mbox{srl}} = 1 \quad\quad \beta_{\mbox{std}} = 1+\tfrac{(b-1)(n\sigma^2-1)}{
 \max\{1, n-1\}
} \quad\quad \beta_{\mbox{dist}}
  =
\tfrac{b}{b-C} (1+\tfrac{(b-C)(n \sigma^2-1)}{\max\{C,n-C\}})
\end{equation}
for the serial, standard and distributed sampling schemes
respectively, with overall mini-batch size $b$ and distribution over
$C$ machines.  Using the weights $v_i=\beta$, we then have the
following obtain the following iteration complexities:
 
{\bf $(1/\gamma)$-Smooth Loss.} 
In this case 
\eqref{eq:asfewfwa}
in Theorem \ref{thm:dualityGapForSmooth}
becomes
 $
 T \geq     \tfrac{ 
\beta
 }b
     (
       \tfrac1{\lambda   \gamma }
    +
    \tfrac{n}   { \beta }
     )
    \log  (
    \tfrac{ \beta }b
      (
     \tfrac1 {\lambda   \gamma }
    +
    \tfrac{n }  { \beta }
     )
    \tfrac1{\epsilon_\calG}
     ).
$
 and hence the iteration complexity is (ignoring logarithmic terms):
$\tilde{\mathcal{O}}\left(
  \tfrac{n}   { b}+
      \tfrac{\beta  
 }b \tfrac1{\lambda   \gamma }
\right).$

{\bf $L$-Lipschitz Continuous Loss.} 
Combining equations
\eqref{eq:dualityRequirements} 
and
\eqref{eq:dualityRequirements3}, and again ignoring logarithmic
factors, we get an iteration complexity of:
\begin{equation}
  \label{eq:liprate}
\tilde{\mathcal{O}}\left(
\tfrac nb 
    + \tfrac{\beta}{b} \tfrac{L^2}{\lambda \epsilon_\calG}
\right).
\end{equation}
Plugging in $\beta_{std}$ into \eqref{eq:liprate} recovers the previous
analysis of Lipschitz loss with standard sampling.

Both the Lipschitz and smooth cases involve two terms: the first term,
$\tfrac{n}{b}$, always displays a linear improvement as we increase
the mini-batch size.  However, in the second term, we also have a
dependence on the data-dependent $1\leq \beta\leq b$, which depends on
the mini-batch size $b$.  We will have a linear improvement in the
second term, i.e.~potential for linear speedup, as long as
$\beta=O(1)$.  For standard sampling we have that $\beta \approx
1+b\sigma^2$, and so we obtain linear speedups as long as
$b=O(1/\sigma^2)$, as discussed in \eqref{eq:liprate}.  We can now
also quantify the effect of distributed sampling and see that it is
quite negligible and yields almost the same speedups and the same
maximum allows mini-batch size as with standard sampling.  Note that
typically we will have $C \ll b$, as we would like to
process multiple example on each machine---otherwise communication
costs would overwhelm computational costs
\cite{shamir2014distributed}.  The analysis supports this choice as
well as the extreme choice $C=b$.
 
Focusing on the smooth loss, it is possible to obtain a linear
reduction in the iteration complexity (corresponding to linear
speedups) for SGD with mini-batch size of up to $\mathcal{O}(\sqrt n
)$ without any data-dependent assumption, that is regardless of the
value of $\beta$
\cite{dekel2012optimal,agarwal2011distributed,cotter2011better}.  Is
this possible also with SDCA?  Indeed, even if we don't account for
the data dependent quantity $\beta$, since we always have $\beta\leq
b$, then the iteration complexity of SDCA for mini-batch SDCA with
smooth loss is: $ \mathcal{O}(1/(\lambda \gamma) + n/b) \log
(1/\epsilon))$ a larger mini-batch scales the second term
(unconditional on any data dependence), and as long as it is the
dominant term, we get linear speedups.  Now, to get the min-max
learning guarantee, we need to set $\lambda=\Theta(1/\sqrt n ) $(see
\cite{shalev2008svm}).  Plugging this in, we see that we get linear
speedups up to a mini-batch of size $\mathcal{O}(\gamma \sqrt n )$.
Unsurprising, this is the same as the mini-batch SGD guarantee.  Now,
if we do take data-dependence into account, we have
$\beta=\mathcal{O}(1+b \sigma^2)$ (where $\sigma^2$ is as defined 
above).  As long as $b<1/\sigma^2$, we get
linear speedps even if the $1/\lambda$ term is dominant, i.e.
regardless of the scaling of lambda relative to $n$.  This is good,
because in practice, and especially when the expected error is low,
the best lambda is often closer to $1/n$ and not $1/\sqrt n $.
Returning to the worst-case rate and $\lambda=1/\sqrt m$ : we now have
an allowed mini-batch size of up to $b=\mathcal{O}(\gamma \sqrt n
/\sigma^2)$ while still getting linear scaling.  That is, we can
combined the benefits of both smoothness, where we can scale the
mini-batch size by $\sqrt{n}$, and the data dependence, to get an
additional scaling by $1/\sigma^2$.

\section{Comparison with CoCoA+}
 
CoCoA+ \cite{ma2015adding} is a recently presented framework and
analysis for distributed optimization of the dual \eqref{eq:dualOfP}:
Data (and hence dual variables) are partitioned among $C$ machines (as
in our distributed sampling), defining $C$ subproblems, one for each
machine.  At each iteration, the set of dual variables of each of the
$C$ machines are updated independently, and then communicated and
aggregated across machines.  Different local updates can be used, and
the CoCoA+ analysis depends on how well the update improves the local
subproblem.  Here we will consider using local SDCA updates in
conjunction with CoCoA+: at each iteration, on each of the $C$
machines, $b/C$ dual variables are selected (as in our distributed
sampling), and $H$ iterations of SDCA are performed sequentially on
these $b/C$ points (in parallel on each of the $C$ machines, and while
considering all other dual variables, including all variables on other
machines, as fixed).

We will consider for simplicity $1$-smooth loss functions and compare
the CoCoA+ guarantees on the number of required iterations
\cite{ma2015adding} to the SDCA gurantees we present here, noting also
the differences in the amount of computation per iterations.  In all
our comparisons, the required communication in each iteration of SDCA
and CoCoA+ is identical and amounts to a single distributed averaging
of vectors in $\R^d$.

{\bf Setting $b=C$ and $H=1$}, we exactly recover mini-batch SDCA with a
minibatch of size $b$, and so we would expect the CoCoA+ analysis to
yield the same guarantee.  However, our guarantee on the number of
required iterations in this case is (ignoring log factors)
$
  \tilde{O}\left(\frac{n}{b} + \frac{1}{b \lambda} +
      \frac{\sigma^2}{\lambda} \right)
$
compared to the CoCoA+ guarantee (ignoring log factors):
$
  \tilde{O}\left(\tfrac{n}{b} + \tfrac{n\tilde\sigma^2}{b \lambda }
+ \tfrac{1}{\lambda} + \tfrac{\tilde\sigma^2}{\lambda^2}\right)
$, where  $\tilde\sigma^2 
= \max_c 
\max_{\alpha : \sum_{i \in \mathcal{P}_c} \|\alpha_i x_i\|^2=1 }
\left(
\tfrac Cn
\| \sum_{i \in  \mathcal{P}_c} \alpha_i x_i \|\right)
 \geq  \sigma^2  \geq  1/n$.  
Our guarantee therefore dominates that of CoCoA+: the second term
is worse by a factor of $n\tilde\sigma^2>1$, the third by a factor of
$1/\sigma^2<1$ and the fourth term in the CoCoA+ bound, can be
particularly bad when $\lambda$ is small (e.g.~when $\lambda\propto 1/n$).

{\bf Setting $b>C$ and $H=b/C$}, both minibatch SDCA and CoCoA+
perform the same number of SDCA updates (same amount of computation)
at each iteration, but while minibatch SDCA's updates are entirely
independent, each group of $H$ CoCoA+ updates (the $H$ updates on the
same machine) are performed sequentially.  We would therefore expect
CoCoA+'s updates to be better, and therefore require less iterations.
Unfortunately, the CoCoA+ analysis does not show this.

To see the deficiency in the CoCoA+ analysis at another extreme,
consider the case where {\bf $b=n$, $1<C<n$ and $H\rightarrow\infty$}.
In this case, each iteration of mini-batch SDCA is actually a full
batch of parallel updates (updating each coordinate independently),
while for CoCoA+ this corresponds to fully optimizing each group of
$n/C$ dual variables using many SDCA updates (and thus much more
computation).  Still, the CoCoA+ iteration bound here would be
$\tilde{O}\left(1+\tfrac{ \sigma' \tilde \sigma^2}{\lambda}\right)$,
where $\sigma'=\max_{\alpha} \tfrac1C \tfrac{\|X^T \alpha\|}{
\sum_{c} \| \sum_{i \in \mathcal{P}_c} x_i \alpha_i\|  }$  and so $\sigma' \tilde\sigma^2\geq \sigma^2$,
compared to the   better mini-batch SDCA bound
$\tilde{O}\left(1+\tfrac{ \sigma^2}{\lambda}\right)$.

And so, even though CoCoA+ with SDCA updates should be a more powerful algorithm, its analysis \cite{ma2015adding} fails to show benefits
over the simpler mini-batch SDCA, and out analysis here of mini-batch
SDCA even dominates the CoCoA+ analysis. The reason for this is that CoCoA+ aims to be a more general framework capable of including arbitrary local solvers. Hence, necessarily, the analysis must be more conservative.

\section{Numerical Experiments}
\label{sec:numerical}

In this Section we show that the cost of distribution is negligible (in terms of \# iterations) when compared to standard mSDCA.
We also show that if $b\gg1$, then CoCoA+ is faster than mSDCA in practice. We have run experiments on 4 datasets (see Table~ \ref{tab:datasets}). Note that most of the datasets are sparse (e.g, news20: an average tsample depends on 385 features out of 1.3M).
 \begin{table}
\caption{Basic characteristics of datasets; 
obtained from libsvm collection \cite{libsvm}.
}
\label{tab:datasets}
\centering
\small
\begin{tabular}{c||r|r|r|r}
 name & \# train. samples & \# test samples & \# features &  Sparsity
 \\ \hline \hline
\emph{epsilon}
& 400,000 &--& 2,000 & 100\%
 \\
 \emph{rcv1} &  20,242 &677,399 & 47,236 & 
   0.15\%
 \\ 
 \emph{news20} & 15,000 & 4,996 & 1,355,191
 & 0.03\%  
 \\
 \emph{real-sim}
 & 72,309& -- & 20,958 &
  0.24\%
\end{tabular}
\end{table}

{\bf Standard vs. Distributed SDCA.}
Figure \ref{fig:speedupsDist} (top row)
compares standard and distributed SDCA.
Recall that distributed sampling with $\C=1$ and standard mini-batch sampling coincide. On the $x$-axis is the parameter $b$ and on the $y$-axis we plot how much more data-accesses we have to as $b$ or $\C$ grow, to get achieve the same accuracy. We see that the lines are almost identical for various choices of $C$, which implies that the cost of using distributed mSDCA does not affect the number of iterations significantly. This is also supported by the theory
(notice that in  \eqref{eq:beta} we have
$\beta_{\mbox{dist}} / \beta_{\mbox{std}}  \approx 1$). 
Also note that, for news20 for instance, increasing $b$ to $10^4$
implies that the number of data-accesses (epochs) will increase by a factor of 11, which implies that  \# iterations will decrease almost by 1,000 for $b=10^4$ when compared with $b=1$.
\begin{figure} 
 \vskip-0.4cm
\centering

\includegraphics[width=1.3in]{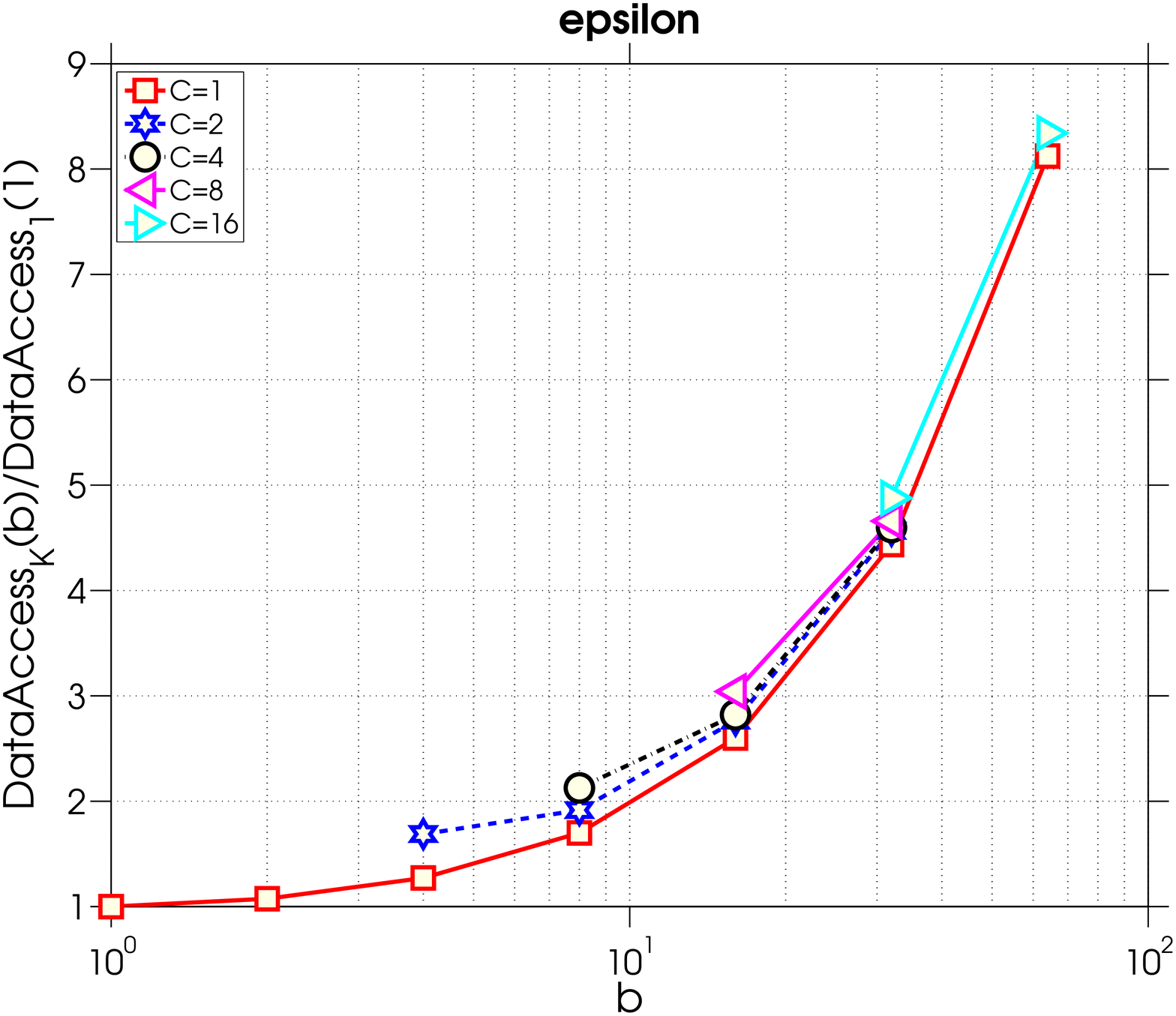}
\includegraphics[width=1.3in]{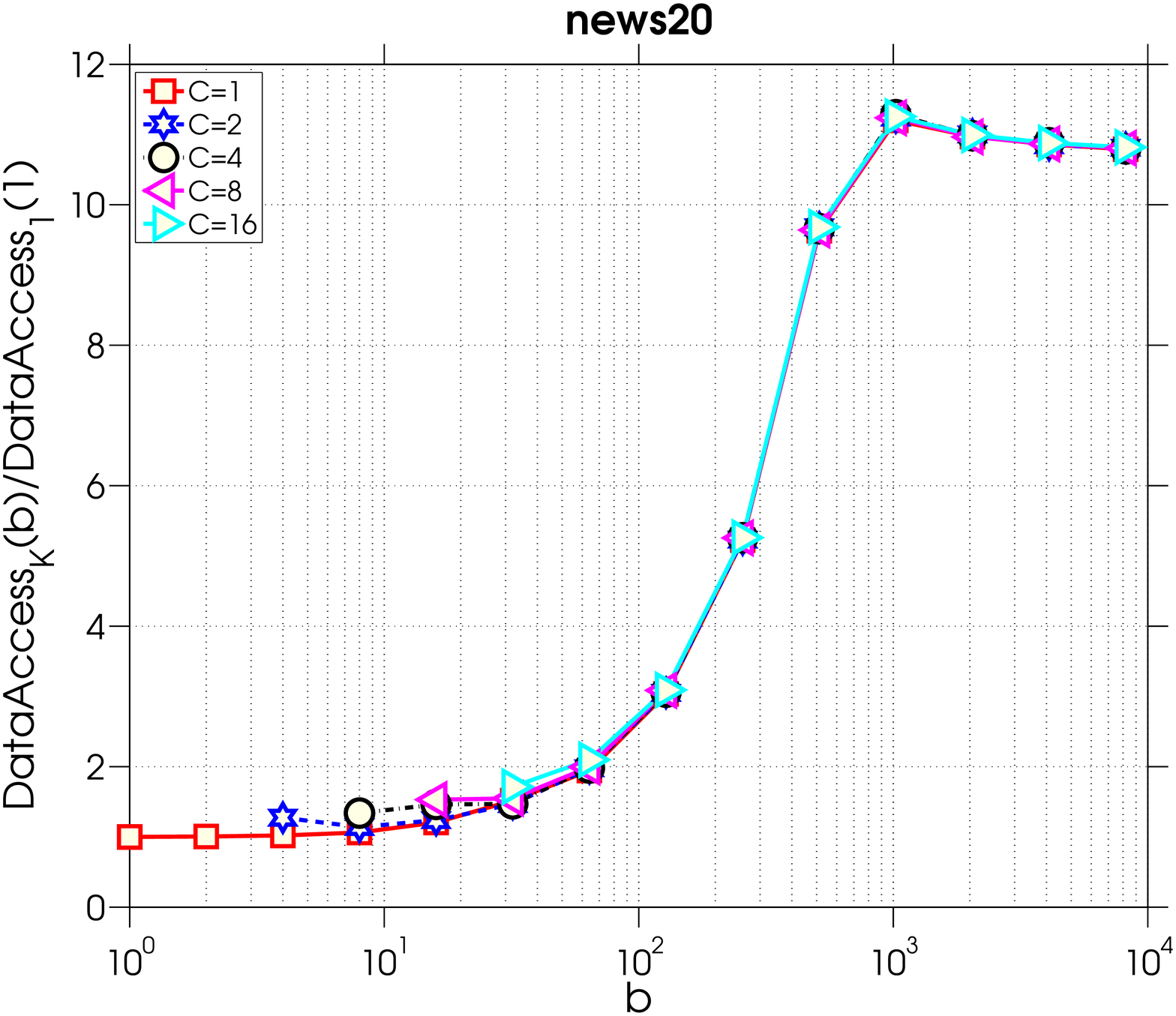}
\includegraphics[width=1.3in]{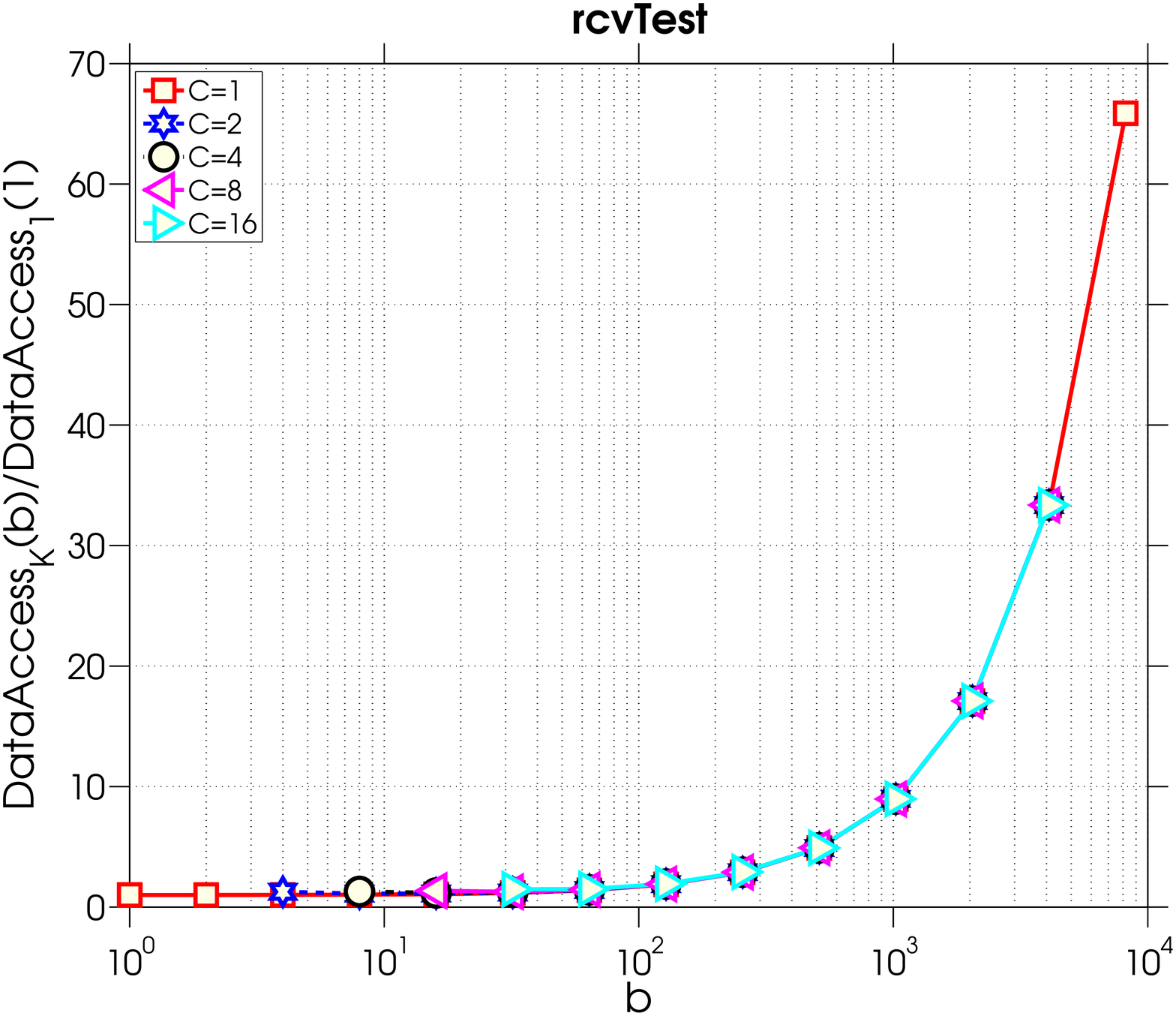}
\includegraphics[width=1.3in]{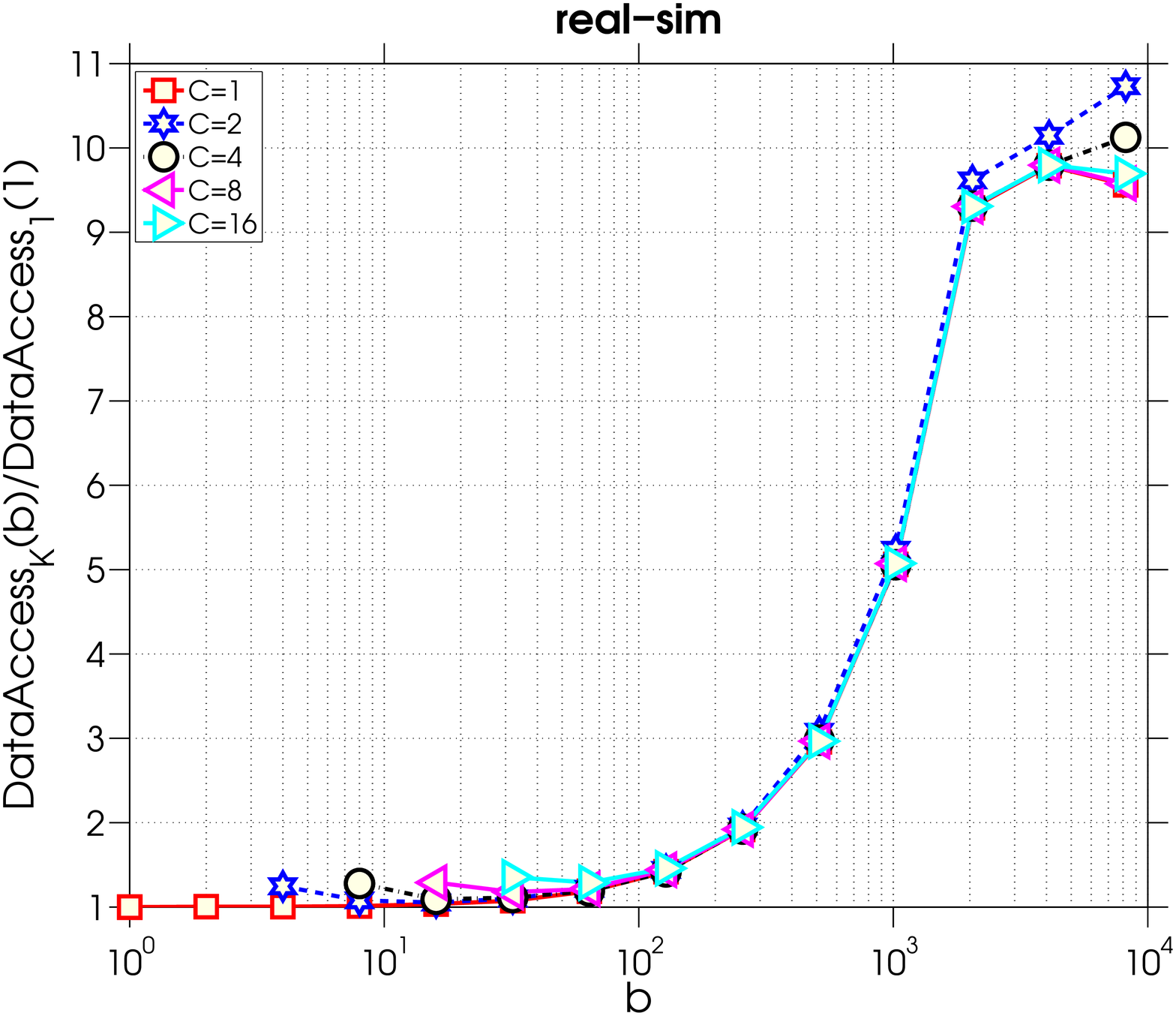}

 \includegraphics[width=1.3in]{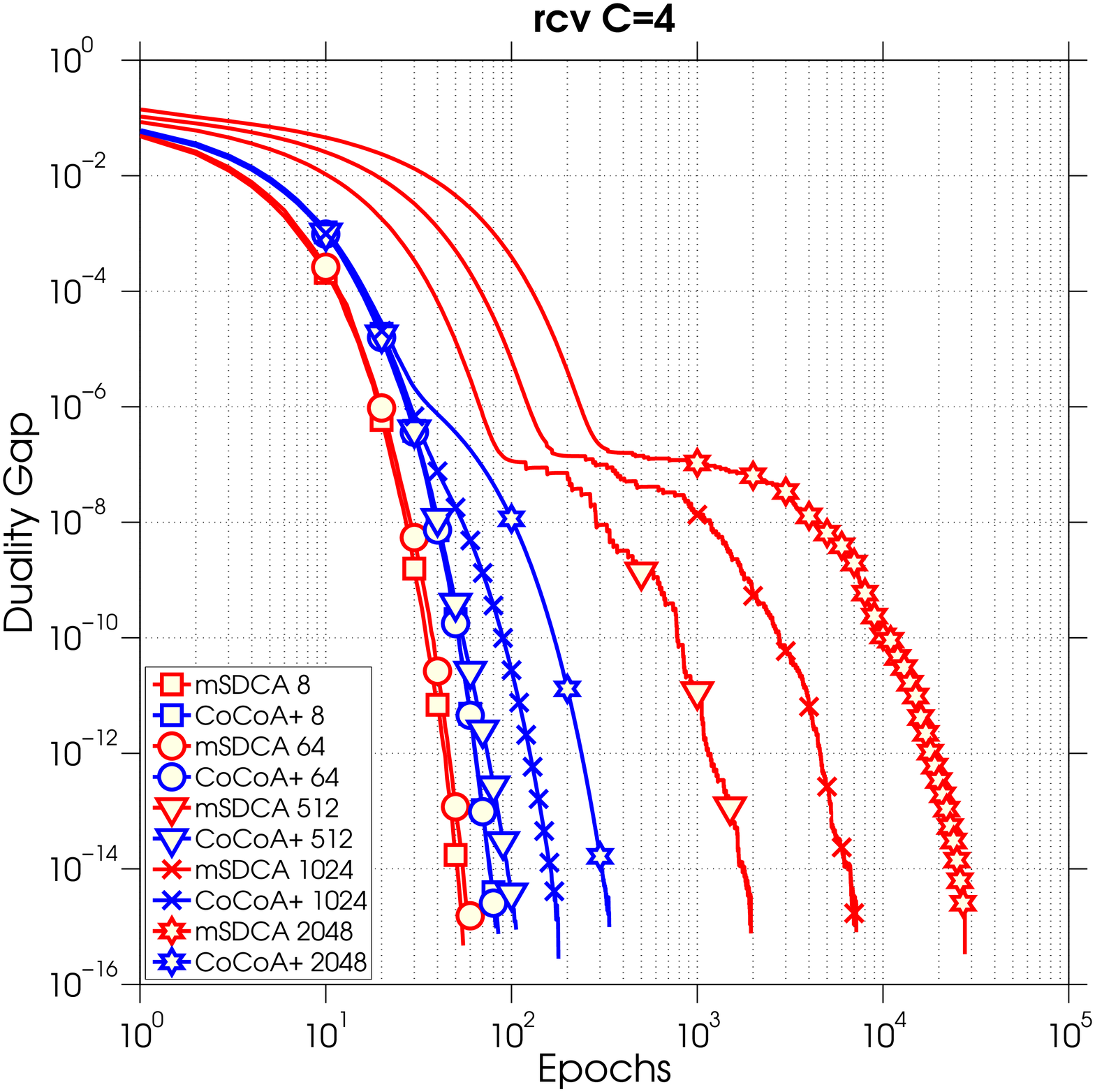}
  \includegraphics[width=1.3in]{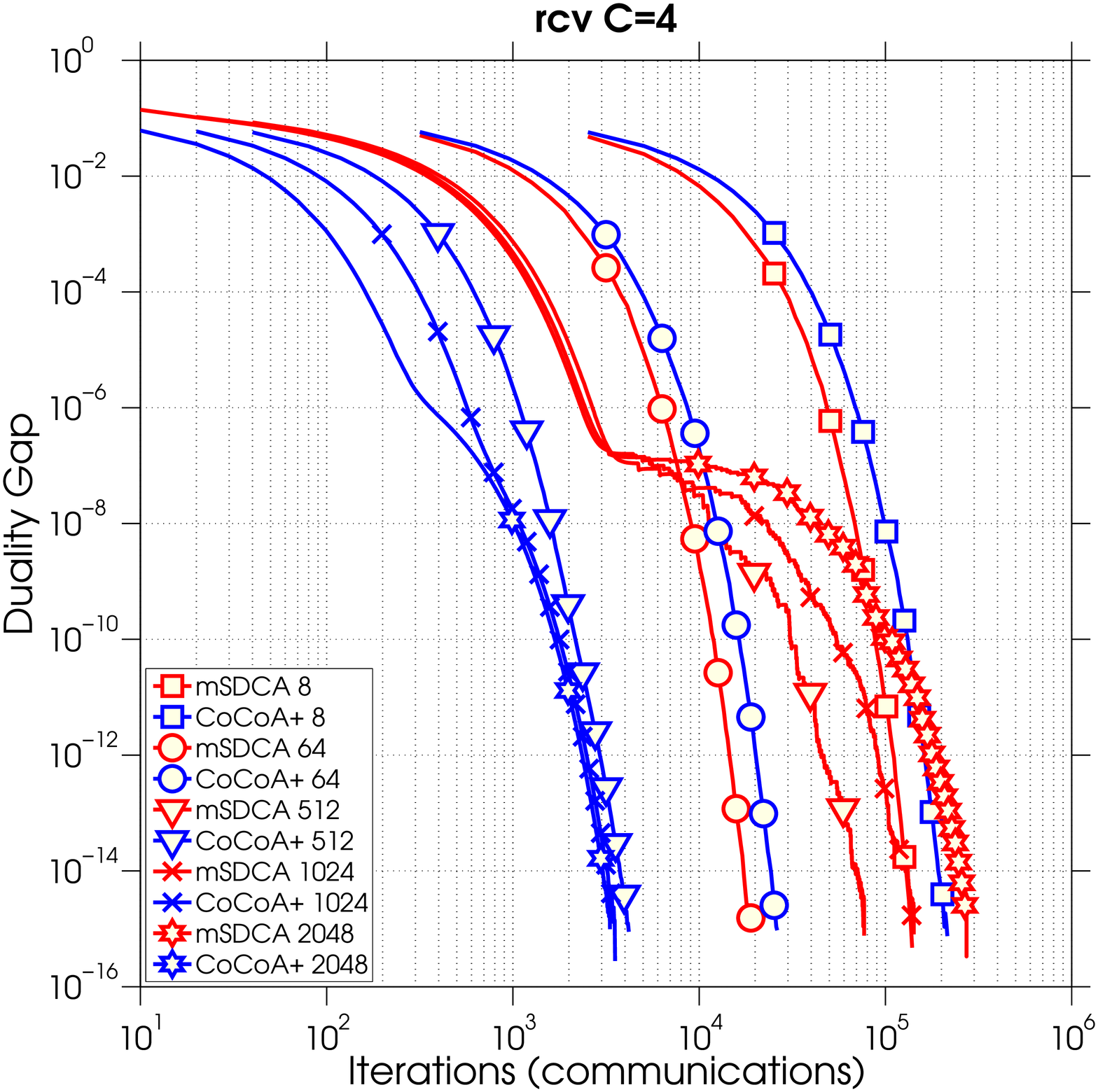}
 \includegraphics[width=1.3in]{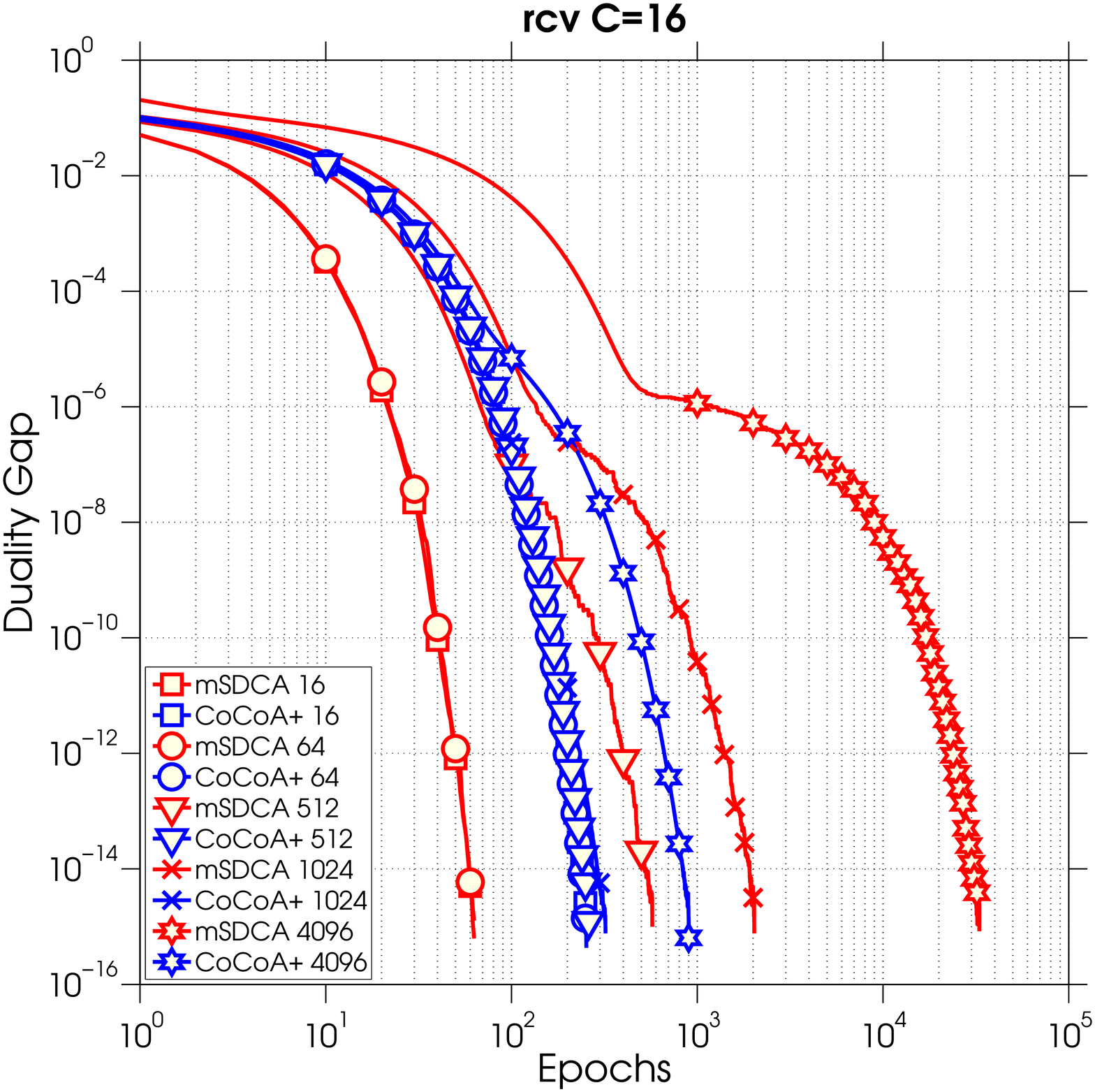}
 \includegraphics[width=1.3in]{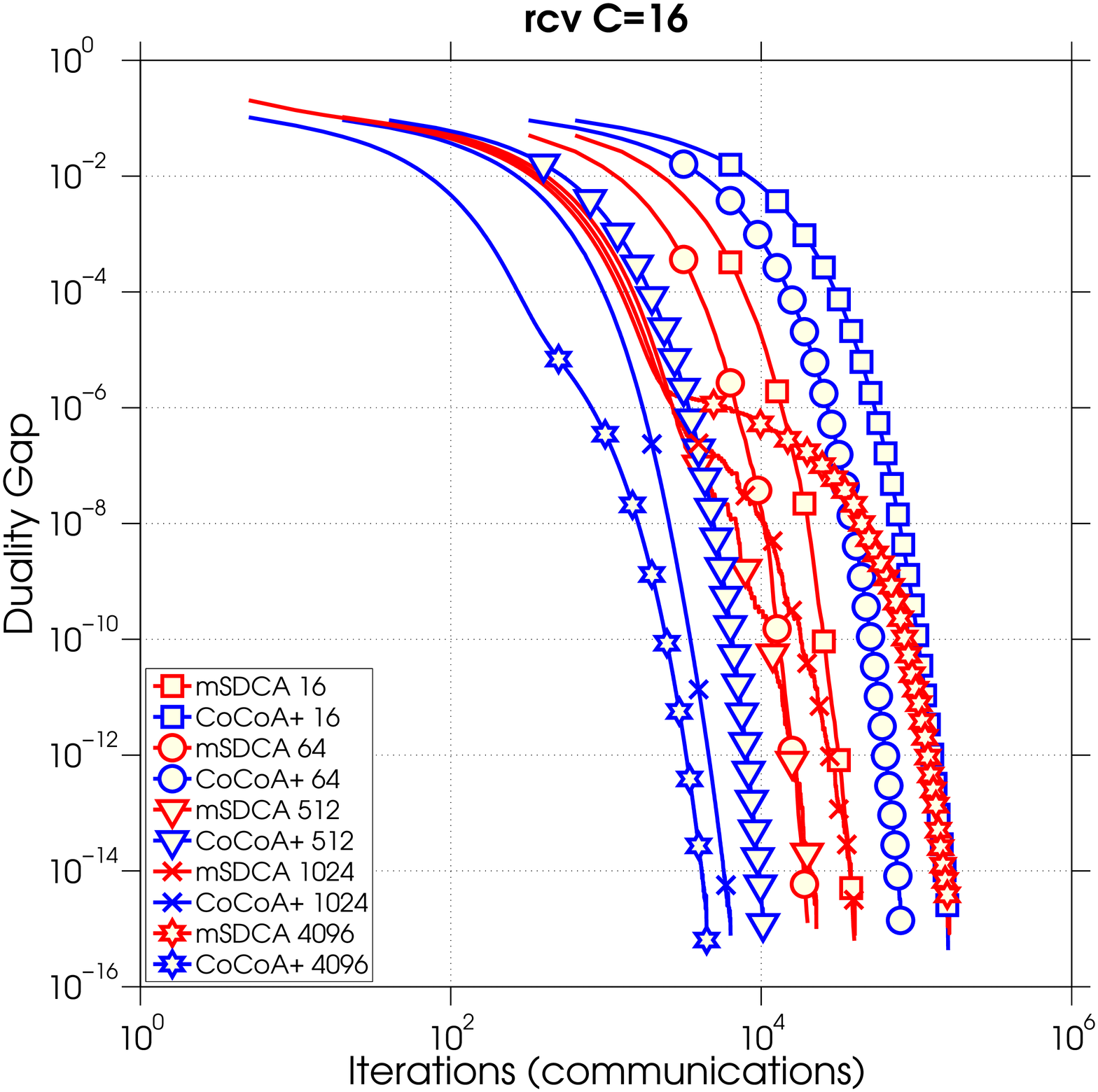} 
  
\caption{\footnotesize  TOP ROW: Number of iterations needed to get an approximate solution is almost the same for standard SDCA and distributed SDCA for $C\in\{1,2,4,8,16\}$. BOTTOM ROW: Comparison of mSDCA and CoCoA+ when solving the SVM dual problem on $C=4$ computers (left) and $C=16$ computers (right).}
\label{fig:speedupsDist}
\end{figure} 

%

{\bf mSDCA vs. CoCoA+.}
 In Figure \ref{fig:speedupsDist} (bottom row)
 we compare the mSDCA with CoCoA+ with SDCA as a local solver. We plot the duality gap as a function of epochs (if communication is negligible then the main cost is in computation) or iterations (if the communication cost is significant than this is the correct measure of performance). As the results suggest, is the communication cost it negligible then the mSDCA with small $b$ is the best (as expected), however, if communications cost is significant, then CoCoA+ with large values of $H$ significantly outperforms mSDCA.

 \clearpage

\bibliographystyle{plain} 
 {\bibliography{minibatch}}

\clearpage
\appendix

\section{Technical Results}

\begin{lemma}[Lemma 2 in \cite{ShalevShawartzZhang}]
\label{lemma:basicValueBound}
 For all $\alf\in\R^n$:
 \begin{align}\label{eq:boundOnD}
  \bD(\alf) \leq \bP(\w^*) \leq \bP({\bf 0}) \leq 1.
 \end{align}
Moreover $\bD({\bf 0}) \geq 0$. 
\end{lemma}

Following Lemma is a minibatch extension of Lemma 1 in \cite{ShalevShawartzZhang}.
\begin{lemma}[Expected increase of dual objective]
\label{lem:basicLemma}
Assume that $\phi_i^*$ is $\gamma$-strongly convex ($\gamma$ can be also zero).
Then, for any $t$ and any $s\in[0,1]$ we have 
\begin{equation}\label{eq:relationOfDualDecreaseAndDualityGap}
 \Exp[\bD(\vc{\alf}{t+1})-\bD(\vc{\alf}{t})]
\geq
   b   (
  \tfrac {s  }{n} \calG(\vc{\alf}{t})
  -  (\tfrac{s}{n} )^2
       \tfrac{1}{ 2  \lambda }
  \vc{G}{t}
 ),
\end{equation}
where  
\begin{align} \nonumber
 \vc{G}{t} &= \tfrac1n ( 
         \|  \vc{\vu}{t}-\vc{\alf}{t}   \|^2_\V
        -\tfrac{\gamma  \lambda n(1-s)}{s }   \|\vc{\vu}{t}-\vc{\alf}{t}\|^2       
       ) 
\\&= \tfrac1n  \label{eq:defOfG}
 \textstyle{\sum}_{i=1}^n  ( 
        v_i
        -\tfrac{\gamma  \lambda n(1-s)}{s } 
       )  (\vc{u_i}{t}-\vc{\alpha_i}{t})^2,      
\end{align}
$\vu_t=(\vc{u_1}{t},\dots,\vc{u_n}{t})^T$ 
and  $-\vc{u_i}{t} \in \partial \phi_i(\w_{\vc{\alf}{t}}^T\x_i)$. 
\end{lemma}

\begin{lemma}[Lemma 3 in \cite{ShalevShawartzZhang}]
\label{lemma:asdkfcpwafcwae}
 Let $\phi:\R\to\R$ be an $L$-Lipschitz continuous. Then for any $|\alpha| > L$
 we have that $\phi^*(\alpha) = \infty$.
\end{lemma}
Following lemma is a small extension of Lemma 4 in \cite{ShalevShawartzZhang}
to obtain more tide bounds  in case each sample has different norm or when ESO bound is used. For example, in serial case we will have
that $\forall t: G^t \leq 4L^2 \frac{\sum_{i=1}^n \|x_i\|^2}{n}$.
\begin{lemma}[Bound on $\vc{G}{t}$]\label{lemma:boundOnGt}
Suppose that for all $i$, $\phi_i$ is $L$-Lipschitz continuous.
Then 
\begin{equation}\label{eq:G_bound}
 \forall t: 
 \vc{G}{t} \leq 4L^2 \tfrac{\textstyle{\sum}_{i=1}^n v_i}{n} .
\end{equation} 
\end{lemma}
\begin{proof}
Indeed,
\begin{align*}
 \vc{G}{t} &\overset{\eqref{eq:defOfG}}{=} 
 \tfrac1n  
 \textstyle{\sum}_{i=1}^n  ( 
        v_i
        -\tfrac{\gamma  \lambda n(1-s)}{s} 
       )  (\vc{u_i}{t}-\vc{\alpha_i}{t})^2
 \overset{\mbox{(Lemma \ref{lemma:asdkfcpwafcwae})}}{\leq} 
 \tfrac1n  
 \textstyle{\sum}_{i=1}^n  ( 
        v_i        
       )  (2L)^2.
      \qedhere
\end{align*}
\end{proof}

\begin{lemma} [Theorem 1 in \cite{richtarik}] \label{l:randomVariableTrick}
Fix $x_0\in \R^N$ and let $\{x_k\}_{k\geq 0}$ be a sequence of random vectors in $\R^N$ with $x_{k+1}$ depending on $x_k$ only. Let $\phi:\R^N\to \R$ be a nonnegative function and define $\xi_k = \phi(x_k)$. Lastly, choose accuracy level $0<\epsilon<\xi_0$, confidence level $0 < \rho< 1$,  and assume that the sequence of random variables $\{\xi_k\}_{k\geq 0}$ is nonincreasing and has one of the following properties:
\begin{enumerate}
\item[(i)] $\E[\xi_{k+1} \;|\; x_k] \leq (1 - \tfrac{\xi_k}{c_1})\xi_k $, for all $k$, where $c_1>\epsilon$ is a constant,
\item[(ii)] $\E[\xi_{k+1} \;|\; x_k] \leq (1-\tfrac{1}{c_2}) \xi_k$, for all $k$ such that $\xi_k\geq \epsilon$, where $c_2>1$ is  a constant.
\end{enumerate}
If property (i) holds and we choose $K \geq 2 + \tfrac{c_1}{\epsilon} (1 - \tfrac{\epsilon}{\xi_0} + \log (\tfrac{1}{\rho}))$, or if property (ii) holds, and we choose
$K\geq c_2 \log (\tfrac{\xi_0}{\epsilon \rho})$, then $\Prob(\xi_K \leq \epsilon) \geq 1-\rho$.
\end{lemma}

\section{Proofs}

\subsection{Proof of Lemma \ref{lem:basicLemma}}
Let us define $\T_\alf$ as an unique maximizer
of a function $\calH(\vt,\alf)$ defined in \eqref{eq:H_definition}, i.e. 
\begin{equation}\label{eq:T_alpha_definition}
 \T_\alf := \arg\max_\vt \calH(\vt,\alf).
\end{equation}
Let us now state some basic properties about function $\calH$. We have that
$\forall \vt,\alf \in \R^n$ and sampling $\hatS$:
\begin{itemize}
 \item $\calH({\bf 0},\alf) = \bD(\alf)$,
 \item from ESO  we have 
 \begin{equation}\label{eq:upperBoundOnExpectedD}
  \Exp[\bD(\alf+\vsubset{\vt}{\hatS})]\geq (1-\tfrac bn) \bD(\alf) + \tfrac bn \calH(\vt,\alf),
 \end{equation}
 \item $\calH(\vt,\alf) \leq  \calH(\T_\alf,\alf)$.
\end{itemize}

Convex conjugate maximal property implies that
\begin{equation}\label{eq:maximalConvexPropertyImplication}
 \phi_i^*(-\vc{u_i}{t})=-\vc{u_i}{t} \w_{\vc{\alf}{t}}^T\x_i - \phi_i(\w_{\vc{\alf}{t}}^T\x_i).
\end{equation}
Let us estimate the expected change of dual objective. 
\begin{align*}
 \frac nb \Exp[\bD(\vc{\alf}{t})-\bD(\vc{\alf}{t+1})]
 &=\frac nb \Exp[\bD(\vc{\alf}{t})-\bD(\vc{\alf}{t}+\vsubset{(\T_{\vc{\alf}{t}})}{\hatS})]
 \overset{\eqref{eq:upperBoundOnExpectedD}}{\leq}
 \bD(\vc{\alf}{t})-\calH(\T_{\vc{\alf}{t}},\vc{\alf}{t})
 \\
 & = 
 \frac1n 
 \sum_{i=1}^n \left(\phi_i^*(-(\alpha_i+\vc{(\T_{\vc{\alf}{t}})}{i})) 
  -\phi_i^*(-\vc{\alpha_i}{t}) \right)
 \\&\quad+ \frac\lambda2 
  \left( 
         \left\|\frac1{\lambda n} \T_{\vc{\alf}{t}} \right\|^2_\V
     +2  \left(\frac1{\lambda n}   \T_{\vc{\alf}{t}} \right) ^T \X \w_\alf \right)
\\& \leq
 \frac1n 
 \sum_{i=1}^n \left(\phi_i^*(-(\vc{\alpha_i}{t}+s(u_i-\vc{\alpha_i}{t})) 
  -\phi_i^*(-\vc{\alpha_i}{t}) \right)
 \\&\quad+ \frac\lambda2 
  \left( 
         \left\|\tfrac1{\lambda n} s(\vu-\vc{\alf}{t}) \right\|^2_\V
     +2  \left(\tfrac1{\lambda n}   s(\vu-\vc{\alf}{t}) \right) ^T \X \w_\alf \right).
\end{align*}
Using $\gamma$-strong convexity of $\phi_i^*$ we have that
\begin{equation}
 \label{eq:stronglyCovnexConjugate}
 \phi_i^*(-(\vc{\alpha_i}{t}+s(u_i-\vc{\alpha_i}{t})) 
  \leq s \phi_i^*(-u_i)
    +(1-s) \phi_i^*(-\vc{\alpha_i}{t})-\tfrac\gamma2 (1-s)s (u_i-\vc{\alpha_i}{t})^2.
\end{equation}
Therefore,
\begin{align*}
\frac nb \Exp[\bD(\vc{\alf}{t})-\bD(\vc{\alf}{t+1})]
& \overset{\eqref{eq:stronglyCovnexConjugate}}{\leq}
 \frac1n 
 \sum_{i=1}^n 
 \left(
    s \phi_i^*(-u_i)
    +su_i \x_i^T\w_{\vc{\alf}{t}}
    -s \phi_i^*(-\vc{\alpha_i}{t})
    -\tfrac\gamma2 (1-s)s (u_i-\vc{\alpha_i}{t})^2
    ) 
 \right)
 \\&\quad+ \frac\lambda2 
  \left( 
         \left\|\tfrac1{\lambda n} s(\vu-\vc{\alf}{t}) \right\|^2_\V       
     +2  \left(\tfrac1{\lambda n}   s(-\vc{\alf}{t}) \right) ^T \X \w_{\vc{\alf}{t}} \right)     
\\& \overset{\eqref{eq:maximalConvexPropertyImplication}}{\leq}
 \frac sn 
 \sum_{i=1}^n 
 \left(
     -u_i \w_{\vc{\alf}{t}}^T\x_i -  \phi_i(\w_{\vc{\alf}{t}}^Tx_i)
    + u_i \x_i^T\w_{\vc{\alf}{t}}
    -  \phi_i^*(-\vc{\alpha_i}{t})
    ) 
 \right)
 \\&\quad+ \frac\lambda2 
  \left( 
  -\frac{\gamma}{\lambda n}  (1-s)s \|\vu-\alf\|^2
  +
         \left\|\tfrac1{\lambda n} s(\vu-\vc{\alf}{t}) \right\|^2_\V
     +2  \left(\tfrac1{\lambda n}   s(-\vc{\alf}{t}) \right) ^T \X \w_{\vc{\alf}{t}} \right).          
\end{align*}
Substituting the definition of duality gap \eqref{eq:dualityGap} we obtain
\begin{align*}
 \frac nb \Exp[\bD(\vc{\alf}{t})-\bD(\vc{\alf}{t+1})]
& \leq
 \frac sn 
 \sum_{i=1}^n 
 \left(
      -  \phi_i(\w_{\vc{\alf}{t}}^T\x_i)    
    -  \phi_i^*(-\vc{\alpha_i}{t})
    - \vc{\alpha_i}{t} \w_{\alf}^T \x_i
    ) 
 \right)
 \\&\quad+ \frac\lambda2 
  \left( 
  -\frac{\gamma}{\lambda n}  (1-s)s \|\vu-\alf\|^2+
         \left\|\tfrac1{\lambda n} s(\vu-\vc{\alf}{t}) \right\|^2_\V
      \right)          
\\&=
  - s \calG(\vc{\alf}{t})
  + \frac\lambda2 
  \left( 
         \left\|\tfrac1{\lambda n} s(\vu-\vc{\alf}{t}) \right\|^2_\V
       -\frac{\gamma}{\lambda n}  (1-s)s \|\vu-\alf\|^2       
      \right).                
\\&=
  - s \calG(\vc{\alf}{t})
  + \frac1{2 \lambda }
  \left(\frac sn\right)^2
  \left( 
         \left\|(\vu-\vc{\alf}{t}) \right\|^2_\V
       -\frac{\gamma n \lambda (1-s)}{s}   \|\vu-\alf\|^2       
      \right).          
\end{align*}
Multiplying both sides by $-\frac bn$ we obtain \eqref{eq:relationOfDualDecreaseAndDualityGap}.

\subsection{Proof of Theorem \ref{thm:dualityGapForLipFunctions}}
At first let us estimate expected change of dual feasibility.
\begin{align}\nonumber
  \Exp[\vc{\epsilon_D}{t+1}]
  &\overset{\eqref{eq:relationOfDualDecreaseAndDualityGap}}{=}
 \Exp[\bD(\alf^*)-\bD(\vc{\alf}{t+1})]
 =
\Exp[\bD(\alf^*)-\bD(\vc{\alf}{t+1})+\bD(\vc{\alf}{t})-\bD(\vc{\alf}{t})]
\\
\nonumber
&
 =
\Exp[\bD(\vc{\alf}{t})-\bD(\vc{\alf}{t+1})+\vc{\epsilon_D}{t}]
\overset{\eqref{eq:relationOfDualDecreaseAndDualityGap},\eqref{eq:G_bound}}{\leq}
 - b \left (
  \frac {s  }{n} \calG(\vc{\alf}{t})
  - \left(\frac{s}{n}\right)^2
       \frac{1}{ 2  \lambda }
G\right)
+
\Exp[\vc{\epsilon_D}{t}]
\\
& \leq
- b
  \frac {s  }{n} \Exp[\vc{\epsilon_D}{t}]
+ b    \left(\frac{s}{n}\right)^2
       \frac{1}{ 2  \lambda }G 
       +
\Exp[\vc{\epsilon_D}{t}]
=
(1- b
  \tfrac {s  }{n} )\Exp[\vc{\epsilon_D}{t}]
+ b    \left(\frac{s}{n}\right)^2
       \frac{ }{ 2  \lambda }G. 
         \label{eq:expectedBound}
\end{align} 
From the above follows that
\begin{align}\label{eq:af09wiw90jgoje2qfw2f}
  \Exp[\vc{\epsilon_D}{t}]
&\leq  
(1- b
  \tfrac {s  }{n} )^t
  \epsilon_D^{0}
+
b    \left(\frac{s}{n}\right)^2
       \frac{ }{ 2  \lambda }G
\sum_{i=0}^{t-1} (1- b
  \tfrac {s  }{n} )^i
\leq  
(1- b
  \tfrac {s  }{n} )^t
  \epsilon_D^{0}
+
    \left(\frac{s}{n}\right)
       \frac{  G}{ 2  \lambda }.
\end{align} 
Choice of 
$s=1$ and $t= t_0:= \max\{0,\lceil \tfrac nb \log(2\lambda n \vc{\epsilon_D}{0} / (G )) \rceil\}$
will lead to 
\begin{align}\label{eq:induction_step1}
  \Exp[\epsilon_D^{t_0}]
 &\leq  
(1- 
  \tfrac {b  }{n} )^{t_0}
  \vc{\epsilon_D}{0}
+
     \frac{s}{n} 
       \frac{  G}{ 2  \lambda }
\leq  
\frac{G }{2\lambda n \vc{\epsilon_D}{0}} \vc{\epsilon_D}{0}
+
    \frac{1}{n}
       \frac{  G}{ 2  \lambda }
=
       \frac{  G}{ \lambda n   }.
\end{align} 
Following the proof in \cite{ShalevShawartzZhang} we are now going to show that
\begin{equation}\label{eq:expectationOfDualFeasibility}
\forall t\geq t_0 :  \Exp[\vc{\epsilon_D}{t}] \leq \frac{2   G}{\lambda (2n+b(t-t_0))}.
\end{equation}
Clearly, \eqref{eq:induction_step1} implies that \eqref{eq:expectationOfDualFeasibility} holds for $t=t_0$.
Now imagine that it holds for any $t\geq t_0$ then we show that it also has to hold for $t+1$. Indeed, using $s=\frac{2n}{2n+b(t-t_0)}$ we obtain
\begin{align}
\Exp[\epsilon_D^{(t+1)}]
&\overset{\eqref{eq:expectedBound}}{\leq}
(1- b
  \tfrac {s  }{n} )\Exp[\vc{\epsilon_D}{t}]
+ b    \left(\frac{s}{n}\right)^2
       \frac{1}{ 2  \lambda }G
\nonumber
\\
&\overset{\eqref{eq:expectationOfDualFeasibility}}{\leq}
(1- b
  \tfrac {s  }{n} ) \frac{2   G}{\lambda (2n+b(t-t_0))}
+ b    \left(\frac{s}{n}\right)^2
       \frac{1}{ 2  \lambda }G
\nonumber
\\
&=
(1- b
  \frac{2}{2n+b(t-t_0)} ) \frac{2   G}{\lambda (2n+b(t-t_0))}
+ b    \left(\frac{2}{2n+b(t-t_0)}\right)^2
       \frac{1}{ 2  \lambda }G   
\nonumber
\\
&=
\frac{2G }{\lambda}
\left(
\frac{1}{2n+b(t-t_0)+b}
\right)       
\left(
\frac{2n+b(t-t_0)+b}{1}
\right)       
\left(
\frac{2n+b(t-t_0)-b}{(2n+b(t-t_0))^2}
\right)       
\nonumber
\\
&=
\frac{2G }{\lambda (2n+b(t-t_0)+b)}
\frac{(2n+b(t-t_0)+b)(2n+b(t-t_0)-b)}{(2n+b(t-t_0))^2}
\nonumber
\\
&\leq \label{eq:fa2f2ffvgafda}
\frac{2G }{\lambda (2n+b(t-t_0)+b)}.
\end{align}
 In the last inequality we have used the fact that geometric mean
 is less or equal to arithmetic mean. 
If $\bar \alf$ is defined as \eqref{eq:averageOfAlphaDefinition}
then we obtain that
\begin{align}
\Exp[\calG(\bar\alf)] &=  
 \Exp\left[\calG\left(\sum_{t=T_0}^{T-1} \tfrac1{T-T_0} \vc{\alf}{t}\right)\right]
 \leq
  \tfrac1{T-T_0} \Exp\left[\sum_{t=T_0}^{T-1} \calG\left( \vc{\alf}{t}\right)\right]
\nonumber  
\\&\overset{\eqref{eq:relationOfDualDecreaseAndDualityGap}}{\leq}
  \tfrac1{T-T_0} \Exp\left[\sum_{t=T_0}^{T-1} 
    \left(  -\frac ns \frac1b \Exp[\bD(\vc{\alf}{t})-\bD(\vc{\alf}{t+1})]
 + \left(\frac{s}{n}\right)
       \frac{1}{ 2  \lambda }
  \left( \tfrac1n
        \left\|  \vu_t-\vc{\alf}{t}  \right\|_\V^2
      \right)
\right)\right]
\nonumber
\\&\overset{\eqref{eq:G_bound}}{\leq}
  \frac ns \frac1b \frac1{T-T_0}    
    \left(   \Exp[\bD(\vc{\alf}{T})] -\Exp[\bD(\vc{\alf}{T_0})]
\right) 
+ \frac{s}{n}
       \frac{G  }{ 2  \lambda } 
\nonumber       
\\&\leq
  \frac ns \frac1b \frac1{T-T_0}    
    \left(  \bD(\alf^*) - \Exp[\bD(\vc{\alf}{T_0})] 
\right) 
+ \frac{s}{n}
       \frac{G  }{ 2  \lambda }. 
\label{eq:asfdj09iuoi2ej4fr3}       
\end{align}
Now, if $T\geq \lceil\frac nb\rceil+T_0$ such that $T_0\geq t_0$
we obtain
\begin{align*}
\Exp[\calG(\bar\alf)] 
&\overset{\eqref{eq:expectationOfDualFeasibility}}{\leq}
  \frac ns \frac1b \tfrac1{T-T_0}    
    \left( \frac{2   G}{\lambda (2n+b(T_0-t_0))}
\right) 
+ \frac{s}{n}
       \frac{G  }{ 2  \lambda } 
\\&=
\frac{  G}{\lambda}
\left(
  \frac ns  \frac1{b(T-T_0)}    
    \left( \frac{2 }{(2n+b(T_0-t_0))}
\right) 
+ \frac{s}{2n}        
\right).       
\end{align*}
Using $s=\frac{n}{b(T-T_0)}$ we obtain that
\begin{align*}
\Exp[\calG(\bar\alf)] 
&\leq
\frac{  G}{b \lambda}
\left(
    \frac{2 }{2\frac nb+ (T_0-t_0)}
+ \frac{1}{2(T-T_0)}
\right).       
\end{align*}
To have this quantity   $\leq \epsilon_\calG$ we obtain that $T,t_0,T_0$ has to satisfy   \eqref{eq:dualityRequirements}.
The fact that 
$T_0 \geq  t_0+
\frac{1}b
\left(
\frac{4 G} { \lambda \epsilon_\calG} -2 n \right)_+$
implies that right-hand site of \eqref{eq:fa2f2ffvgafda}
is $\leq \epsilon_\calG$.

\subsection{Proof of Theorem \ref{thm:dualityGapForSmooth}}
If function $\phi_i$ is $(1/\gamma)$-smooth then $\phi_i^*$ is $\gamma$-strongly convex.
If we plug $s=\tilde s=\frac{\lambda n \gamma }{  \|\vv\|_\infty+\lambda n \gamma } \in (0,1)$ 
into \eqref{eq:defOfG}
we obtain that $\forall t: \vc{G}{t} \leq 0$.
Hence \eqref{eq:relationOfDualDecreaseAndDualityGap}
will read as follows
\begin{align}\label{eq:gdceef23fr}
 \Exp[\bD(\vc{\alf}{t+1})-\bD(\vc{\alf}{t})]
\geq
  b 
  \frac {\tilde s  }{n} \calG(\vc{\alf}{t})
  =
    b 
   \frac{\lambda   \gamma }{  \|\vv\|_\infty+\lambda n \gamma }    \calG(\vc{\alf}{t})
  \geq  
    b 
   \frac{\lambda   \gamma }{  \|\vv\|_\infty+\lambda n \gamma }    (\bD(\alf^*)-\bD(\vc{\alf}{t})).
  \end{align}
Using the fact that
$\Exp[\bD(\vc{\alf}{t+1})-\bD(\vc{\alf}{t})]
=\Exp[\bD(\vc{\alf}{t+1})-\bD(\alf^*)]
+\bD(\alf^*)-\bD(\vc{\alf}{t})
$
we have 
\begin{align}\label{eq:smoothDecreaseInExpectation}
\Exp[\bD(\alf^*)-\bD(\vc{\alf}{t+1})]
  \leq  
  \left(1 - b 
   \frac{\lambda   \gamma }{  \|\vv\|_\infty+\lambda n \gamma }   
   \right)(\bD(\alf^*)-\bD(\vc{\alf}{t})).
   \end{align}
Therefore if we denote by $\vc{\epsilon_D}{t} = \bD(\alf^*)-\bD(\vc{\alf}{t})$
we have that
\begin{align*}
 \Exp[\vc{\epsilon_D}{t}] \leq   \left(1 - b 
   \frac{\lambda   \gamma }{  \|\vv\|_\infty+\lambda n \gamma }   
   \right)^t \vc{\epsilon_D}{0}
\overset{\eqref{eq:boundOnD}}{\leq}
\left(1 - b 
   \frac{\lambda   \gamma }{  \|\vv\|_\infty+\lambda n \gamma }   
   \right)^t
\leq \exp\left(-b t
   \frac{\lambda   \gamma }{  \|\vv\|_\infty+\lambda n \gamma }    \right).
\end{align*}
Right hand site will be smaller than some $\epsilon_D$ if 
$$
 t   
    \geq 
    \frac{  \|\vv\|_\infty}b
    \left(
       \frac1{\lambda   \gamma }
    +
    \frac{n}   {  \|\vv\|_\infty}
    \right)
    \log \frac1{\epsilon_D}.
$$
Moreover, to bound the duality gap we have
\begin{align}\label{eq:fca9fiuipojf}
b 
   \frac{\lambda   \gamma }{  \|\vv\|_\infty+\lambda n \gamma }    \calG(\vc{\alf}{t}) 
\overset{\eqref{eq:gdceef23fr}}{\leq}
 \Exp[\vc{\epsilon_D}{t}-\vc{\epsilon_D}{t+1}]   
 \leq \vc{\epsilon_D}{t}.
  \end{align}
Therefore  $\calG(\vc{\alf}{t})\leq    \frac{  \|\vv\|_\infty+\lambda n \gamma }    {b\lambda   \gamma }    \vc{\epsilon_D}{t}$.  
Hence if $\epsilon_D \leq 
\frac{b\lambda   \gamma } {  \|\vv\|_\infty+\lambda n \gamma }   \epsilon_\calG $
then $\calG(\vc{\alf}{t})\leq \epsilon_\calG$.
Therefore
after 
$$
 t   
    \geq 
    \frac{  \|\vv\|_\infty}b
    \left(
       \frac1{\lambda   \gamma }
    +
    \frac{n}   {  \|\vv\|_\infty}
    \right)
    \log \left(
    \frac{  \|\vv\|_\infty}b
     \left(
     \frac1 {\lambda   \gamma }
    +
    \frac{n }  {  \|\vv\|_\infty}
    \right)
    \frac1{\epsilon_\calG}
    \right).
$$
iterations we have duality gap less than $\epsilon_\calG$ and the first part of the proof is done.
To show the second part of Theorem let us sum 
\eqref{eq:fca9fiuipojf} over $t=T_0,\dots,T-1$ to obtain
\begin{align}\label{eq:said9fipdsaf}
\Exp\left[
\frac1{T-T_0} \sum_{t=T_0}^{T-1} \calG(\vc{\alf}{t})
\right] 
\leq
   \frac{  \|\vv\|_\infty+\lambda n \gamma }{b\lambda   \gamma }
   \frac1{T-T_0} \Exp[\bD(\vc{\alf}{T})-\bD(\vc{\alf}{T_0})].
\end{align}
Now, if we choose
$\bar \w, \bar \alf$ to be either average vectors or a randomly chosen vector over $t\in\{T_0+1,\dots,T\}$, then we have
\begin{align*}
\Exp[\calG(\bar \alf)]
\overset{\eqref{eq:said9fipdsaf}}{\leq}
   \frac{  \|\vv\|_\infty+\lambda n \gamma }{b\lambda   \gamma }
   \frac1{T-T_0} \Exp[\bD(\vc{\alf}{T})-\bD(\vc{\alf}{T_0})]
\leq
   \frac{  \|\vv\|_\infty+\lambda n \gamma }{b\lambda   \gamma }
   \frac1{T-T_0} \Exp[\bD(\alpha^*)-\bD(\vc{\alf}{T_0})]. 
   \end{align*}
Hence to have $\Exp[\calG(\bar \alf)]\leq \epsilon_\calG$
it is sufficient to choose
$$
\Exp[\vc{\epsilon_D}{T_0}]\leq 
   \frac{b\lambda   \gamma }{  \|\vv\|_\infty+\lambda n \gamma }
   (T-T_0) \epsilon_\calG.
$$
Therefore we need $T_0$ to satisfy
$$
 T_0   
    \geq 
    \frac{  \|\vv\|_\infty}b
    \left(
       \frac1{\lambda   \gamma }
    +
    \frac{n}   {  \|\vv\|_\infty}
    \right)
    \log 
    \left(
    \frac{  \|\vv\|_\infty}b
    \left(
    \frac{1}{\lambda   \gamma }
    +
    \frac{n }{  \|\vv\|_\infty}
    \right) \frac{1}{(T-T_0)\epsilon_\calG}
    \right).
$$
To get a high probability result we use Lemma  \ref{l:randomVariableTrick}
with 
$\vc{\xi}{t}=\bD(\alf^*)-\bD(\vc{\alf}{t})$, 
$c_2 = \frac{  \|\vv\|_\infty+\lambda n \gamma }   {b \lambda   \gamma } $ (see \eqref{eq:smoothDecreaseInExpectation}) and 
$\epsilon=\frac{b\lambda   \gamma }{  \|\vv\|_\infty+\lambda n \gamma }    \epsilon_\calG$
to obtain that
after 
$$\tilde T= c_2 \log\left(\frac{\vc{\xi}{0}}{\epsilon\rho}\right)
  \overset{\mbox{Lemma}\ \ref{lemma:basicValueBound}}{\leq} \frac{  \|\vv\|_\infty+\lambda n \gamma }   {b \lambda   \gamma }
  \log\left(\frac{1}{\epsilon\rho}\right)
$$
$$
1-\rho \leq \Prob\left( \bD(\alf^*)-\bD(\vc{\alf}{\tilde T}) \leq \epsilon\right)
\overset{\eqref{eq:fca9fiuipojf}}{\leq}
\Prob\left( \frac{b\lambda   \gamma }{  \|\vv\|_\infty+\lambda n \gamma }    \calG(\vc{\alf}{t})   \leq \epsilon\right)
=
\Prob\left(  \calG(\vc{\alf}{t})   \leq \epsilon_\calG\right).
$$


  \end{document}